%% file: main.tex
\PassOptionsToPackage{dvipsnames}{xcolor}
\documentclass{article}

\usepackage{microtype}
\usepackage{graphicx}
\usepackage{subfigure}
\usepackage{booktabs} 

\usepackage[pagebackref]{hyperref}

\usepackage[accepted]{icml}

\hypersetup{
    colorlinks=true,
    linkcolor=red,
    citecolor=cyan,
    filecolor=magenta,      
    urlcolor=magenta,
}

\usepackage{amsmath}
\usepackage{amssymb}
\usepackage{mathtools}
\usepackage{amsthm}

\DeclareFontFamily{OT1}{pzc}{}
\DeclareFontShape{OT1}{pzc}{m}{it}{<-> s * [1.10] pzcmi7t}{}
\DeclareMathAlphabet{\mathpzc}{OT1}{pzc}{m}{it}

\usepackage{xspace}
\newcommand{\model}[0]{\textsc{KnowFormer}\xspace}

\usepackage{enumitem}
\usepackage{multirow}
\usepackage{multicol}
\usepackage{adjustbox}
\usepackage{colortbl} 
\usepackage{xcolor}
\usepackage{arydshln}
\usepackage{array} 
\newcommand*{\ldblbrace}{\{\mskip-5mu\{}
\newcommand*{\rdblbrace}{\}\mskip-5mu\}}

\usepackage[capitalize,noabbrev]{cleveref}

\usepackage{thm-restate}
\theoremstyle{plain}
\newtheorem{theorem}{Theorem}[section]

\newtheorem{lemma}[theorem]{Lemma}

\theoremstyle{definition}
\newtheorem{definition}[theorem]{Definition}

\theoremstyle{remark}

\usepackage[textsize=tiny]{todonotes}

\icmltitlerunning{\model: Revisiting Transformers for Knowledge Graph Reasoning}

\begin{document}

\twocolumn[
\icmltitle{\model: Revisiting Transformers for Knowledge Graph Reasoning}



\icmlsetsymbol{star}{*}
\icmlsetsymbol{dagger}{$\dagger$}

\begin{icmlauthorlist}
\icmlauthor{Junnan Liu}{yyy,comp,star}
\icmlauthor{Qianren Mao}{yyy,dagger}
\icmlauthor{Weifeng Jiang}{sch}
\icmlauthor{Jianxin Li}{yyy,comp,dagger}
\end{icmlauthorlist}

\icmlaffiliation{yyy}{Zhongguancun Laboratory, Beijing, P.R.China.}
\icmlaffiliation{comp}{SCSE, Beihang University, Beijing, P.R.China.}
\icmlaffiliation{sch}{SCSE, Nanyang Technological University, Singapore}

\icmlcorrespondingauthor{Junnan Liu}{to.liujn@outlook.com}

\icmlkeywords{Machine Learning, ICML}

\vskip 0.3in
]



\printAffiliationsAndNotice{\textsuperscript{*}Work done during the internship at Zhongguancun Laboratory. \textsuperscript{$\dagger$}Corresponding authors.} 

\begin{abstract}
\input{sections/abstract}
\end{abstract}

\section{Introduction}
\input{sections/1.introduction}

\section{Related Work}
\input{sections/2.related}

\section{Preliminary}
\input{sections/3.preliminary}

\section{Proposed Method}
\input{sections/4.method}

\section{Experimental Results}
\input{sections/5.experiment}

\section{Conclusion}
\input{sections/6.conclusion}

\section*{Acknowledgements}
We appreciate all anonymous reviewers for their careful review and valuable comments.
The authors of this paper are supported by the National Natural Science Foundation of China through grant No.62225202.

\section*{Impact Statement}
This paper introduces \model, a novel method for knowledge graph reasoning.
Knowledge graph reasoning is widely applied in various fields, including personalized recommendations, drug discovery, and enterprise-level knowledge management decisions. 
Each of these applications has the potential to generate significant societal impacts. 
Personalized recommendations, for instance, can enhance user experience, but also raise concerns about privacy breaches. 
For example, in a social platform, users can provide personalized friend recommendations based on their personal information, but this information is also easy to cause privacy disclosure.
Similarly, drug discovery contributes to advancements in the field of biomedicine, but it also introduces risks associated with medical accidents.
We encourage researchers to further investigate and address concerns related to privacy risks and errors that may arise in knowledge graph reasoning.


\bibliography{ref}
\bibliographystyle{icml}

\newpage
\appendix
\onecolumn
\input{sections/appendix}


\end{document}

%% file: sections/abstract.tex
Knowledge graph reasoning plays a vital role in various applications and has garnered considerable attention. 
Recently, path-based methods have achieved impressive performance.
However, they may face limitations stemming from constraints in message-passing neural networks, such as missing paths and information over-squashing.
In this paper, we revisit the application of transformers for knowledge graph reasoning to address the constraints faced by path-based methods and propose a novel method \model.
\model utilizes a transformer architecture to perform reasoning on knowledge graphs from the message-passing perspective, rather than reasoning by textual information like previous pretrained language model based methods.
Specifically, we define the attention computation based on the query prototype of knowledge graph reasoning, facilitating convenient construction and efficient optimization. 
To incorporate structural information into the self-attention mechanism, we introduce structure-aware modules to calculate query, key, and value respectively. 
Additionally, we present an efficient attention computation method for better scalability. 
Experimental results demonstrate the superior performance of \model compared to prominent baseline methods on both transductive and inductive benchmarks.

%% file: sections/1.introduction.tex
Knowledge graphs (KGs) are structured knowledge bases that store known facts in the form of triplets~\cite{HoganBCdMGKGNNN21,JiPCMY22}.
Each triplet consists of a head entity, a relation, and a tail entity. 
However, real-world KGs often suffer from high incompleteness, making it challenging to retrieve the desired facts~\cite{WangMWG17,JiPCMY22}.
Knowledge graph reasoning, which involves inferring new facts based on existing ones, is a fundamental and indispensable task in KGs with a wide range of applications~\cite{QiuZFLJLLZ19,0003W0HC19,HuangZLL19,AbujabalRYW18,ChengYWZ020,zeng2022toward}.

\begin{figure*}[htb]
    \centering
    \includegraphics[scale=0.4]{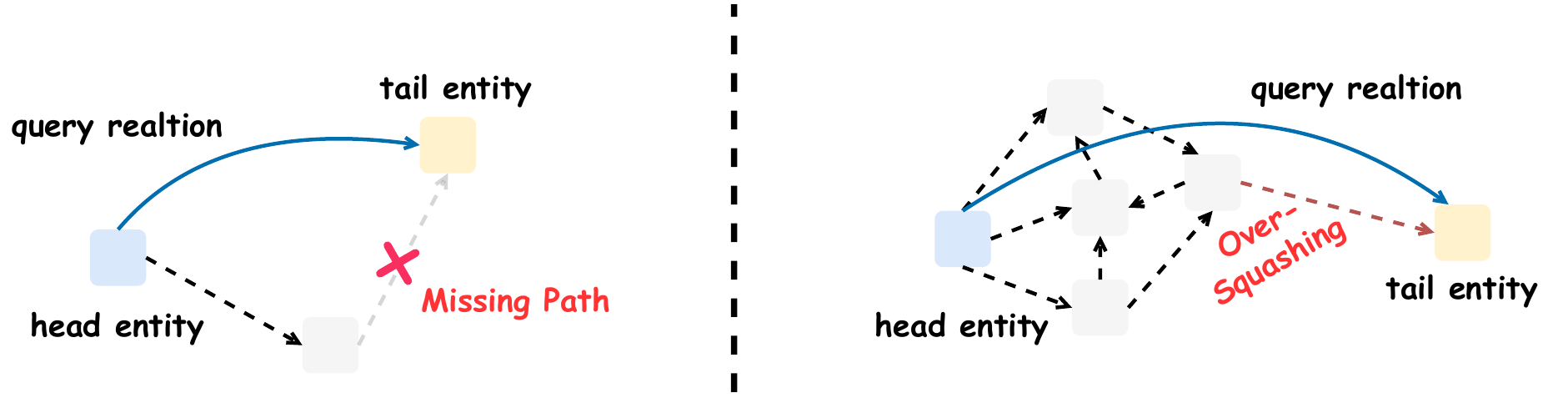}
    \vspace{-0.2cm}
    \caption{Path-based methods could be limited by the missing paths~\cite{FranceschiNPH19} and over-squashing information~\cite{0002Y21}.}
    \label{fig:fig.0}
\end{figure*}

A substantial line of previous research has been dedicated to developing effective and robust methods for knowledge graph reasoning. 
Knowledge graph embedding (KGE) methods focus on embedding entities and relations into a low-dimensional space~\cite{BordesUGWY13,SunDNT19,YangYHGD14a,TrouillonDGWRB17,Li0LH0LSWDSXZ22}.
Despite their impressive performance, these embedding-based methods struggle to generalize to inductive scenarios since they cannot leverage the local structures of knowledge graphs~\cite{TeruDH20}.
As a result, researchers have turned to path-based methods to enhance reasoning performance and improve generalization. 
These methods learn pairwise representations from subgraphs ~\cite{TeruDH20,MaiZY021,WangZPDSS22,ChamberlainSRFM23} or relational paths~\cite{ZhuZXT21,ZhangY22,ZhangZY0023,zhu2023net} between entities which can be used to predict unseen facts.
However, path-based methods can be limited by the missing paths~\cite{FranceschiNPH19} and over-squashing information~\cite{0002Y21} as shown in \cref{fig:fig.0}.

Recent works have leveraged the transformers for knowledge graph reasoning, drawing inspiration from the success of transformer-based models in various domains~\cite{VaswaniSPUJGKP17,DevlinCLT19,DosovitskiyB0WZ21,RivesMSGLLGOZMF21}.
Typically, these studies encode components of knowledge graphs using their text descriptions through pretrained language models (PLMs) and utilize either a discriminative paradigm~\cite{LvL00LLLZ22,0046ZWL22} or a generative paradigm~\cite{XieZLDCXCC22,SaxenaKG22} to predict new facts.
However, there are two main issues: (1) textual descriptions often contain ambiguity and require domain-specific knowledge for accurate encoding, \textit{e.g.}, pretrained language models like BERT~\cite{DevlinCLT19}, which are trained with commonsense knowledge, may not transfer well to domain-specific KGs; and (2) capturing structural information, which can be crucial for knowledge graph reasoning owing to the ~\textit{low-rank assumption}~\cite{KorenBV09,LacroixUO18} is challenging, and some approaches address this by utilizing path and context encoding or pretraining on large-scale KGs as alternatives~\cite{ChenLG0ZJ21,LiWM23}.

In this paper, we aim to further investigate the potential benefits of applying the \textit{transformer architecture}~\cite{VaswaniSPUJGKP17} to the task of knowledge graph reasoning.
Contrary to previous approaches that utilize pretrained language models for encoding text descriptions, we leverage the attention mechanism within the transformer architecture to capture interactions between any pair of entities.
Our principal contribution is the introduction of an expressive and scalable attention mechanism tailored specifically to knowledge graph reasoning. 
The resulting model, named \textbf{\model}, provides informative representations for knowledge graph reasoning in both transductive and inductive scenarios.
It can address the limitations of vanilla path-based methods, such as the path missing and information over-squashing.
Specifically:
\begin{itemize}[itemsep=3pt,topsep=0pt,parsep=1pt]
    \item We redefine the self-attention mechanism, originally introduced by \citet{VaswaniSPUJGKP17}, as a weighted aggregation of pairwise information for specific queries based on the plausibility of entity pairs as query prototypes. 
    This redefinition enables us to perform attention computation on multi-relational KGs and reduce the modeling complexity.
    \item For constructing the attention mechanism, we introduce two modules that generate informative representations for query, key, and value, respectively.
    These modules are designed based on relational message passing neural networks, allowing us to consider structural information during attention calculation.
    \item To improve scalability, we adopt an instance-based similarity measure~\cite{CuiC22} to reduce the significant number of degrees of freedom and introduce an approximation method that maintains complexity linearly proportional to entities. Moreover, we provide theoretical guarantees for its stability and expressivity.
    \item We also provide empirical evidence of the effectiveness of our proposed \model on knowledge graph reasoning benchmarks in both transductive and inductive settings, surpassing the performance of salient baselines.
\end{itemize}

%% file: sections/2.related.tex
We classify the related work on knowledge graph reasoning into three main paradigms: embedding-based methods, path-based methods, and transformer models.

\paragraph{Embedding-Based Methods.}
Embedding-based methods aim to learn distributed representations for entities and relations by preserving the triplets in the knowledge graph.
Notable early methods include TransE~\cite{BordesUGWY13}, DistMult~\cite{YangYHGD14a} and ComplEX~\cite{TrouillonDGWRB17}.
Subsequent work has focused on improving the score function or embedding space of these methods to enhance the modeling of semantic patterns~\cite{SunDNT19,TangHWHZ20,ZhangCZW20,Li0LH0LSWDSXZ22}.
In addition, some researchers have explored the application of neural networks in an encoder-decoder paradigm to obtain adaptive and robust embeddings. 
Representative methods include convolutional neural networks ~\cite{DettmersMS018,NguyenNNP18} and graph neural networks~\cite{SchlichtkrullKB18,VashishthSNT20,YouGYL21}.
Although embedding methods have demonstrated promising performance in knowledge graph reasoning and scalability on large KGs, 
these embeddings are hard to generalize to unseen entities and relations, thus limiting their applicability in the inductive setting.

\paragraph{Path-Based Methods.}
Path-based methods in knowledge graph reasoning have their origins in traditional heuristic similarity approaches, which include measuring the weighted count of paths~\cite{katz1953new}, random walk probability ~\cite{page1998pagerank}, and shortest path length~\cite{Liben-NowellK07}.
In recent years, there have been proposals to employ neural networks to encode paths between entities, such as recurrent neural networks~\cite{NeelakantanRM15}, and aggregate these representations for reasoning. 
Another research direction focuses on learning probabilistic logical rules over KGs~\cite{YangYC17,SadeghianADW19} and utilizing these rules to assign weights to paths. 
More recent works have achieved state-of-the-art performance by incorporating well-designed message passing neural networks~\cite{0004RL21,ZhuZXT21,ZhangY22}.
These methods enhance expressivity through the use of \textit{target distinguishable} initial function and \textit{relational} message and aggregation functions.
PathCon~\cite{0004RL21}, for instance, introduces a relational message-passing framework for the KG completion task and achieves impressive results.
However, prominent path-based methods still face limitations, including the recognized shortcomings of message-passing neural networks such as incompleteness~\cite{FranceschiNPH19} and over-squashing~\cite{0002Y21}. 

\paragraph{Transformers for Knowledge Graph Reasoning.}
Currently, most research utilizes pretrained models based on transformer architecture to encode textual descriptions and leverage its knowledge and the semantic understanding ability for reasoning.
For instance, KG-BERT~\cite{abs-1909-03193} and PKGC~\cite{LvL00LLLZ22} employ a pretrained language model encoder to represent KG triples with text descriptions, enabling the inference of new facts through a classification layer using a special token representation. 
In contrast, some studies explore a sequence-to-sequence approach for generating reasoning results~\cite{XieZLDCXCC22,SaxenaKG22}.
This paradigm has received increasing attention with the emergence of large language models.
To bridge the gap between unstructured textual descriptions and structured KGs, researchers have explored incorporating contextual information by sampling and encoding neighborhood triplets or paths~\cite{XieZLDCXCC22,ChenLG0ZJ21,LiWM23}.
However, these methods face limitations due to their heavy reliance on text description, including the lack of textual data, domain knowledge requirements, and the neglect of structural information in KGs.
In this paper, we revisit how to perform knowledge graph reasoning leveraging the transformer architecture from the perspective of knowledge graph structure.

%% file: sections/3.preliminary.tex
In this section, we introduce the background knowledge of knowledge graphs and knowledge graph reasoning.
Due to the page limitations, more preliminaries about transformers can be found in the appendix.

\paragraph{Knowledge Graph.}
Typically, a knowledge graph $\mathcal{G} = \{\mathcal{V}, \mathcal{E}, \mathcal{R}\}$ is a collection of triplets $\mathcal{E} = \{(h_i,r_i,t_i)\;|\; h_i,t_i \in \mathcal{V}, r_i \in \mathcal{R}\}$ consist a set of entities $\mathcal{V}$ and a set of relations $\mathcal{R}$.
Each triplet is a relational edge from head entity $h_i$ to tail entity $t_i$ with the relation $r_i$.
For ease of notation, we can also represent a fact as $r(u,v) \in \mathcal{E}$ where $u,v \in \mathcal{V}$ and $r\in \mathcal{R}$.
Additionally, we define the neighborhood set of an entity $u \in \mathcal{V}$ relative to a relation $r\in \mathcal{R}$ as $\mathcal{N}_r(u) = \{ v\; |\; r(u,v) \in \mathcal{E} \}$.

\paragraph{Knowledge Graph Reasoning.}
Given a knowledge graph $\mathcal{G} = \{\mathcal{V}, \mathcal{E}, \mathcal{R}\}$, the goal of knowledge graph reasoning is to leverage existing facts to infer the missing elements of the query fact $(h, r_q, t)$, where $r_q$ is a query relation. 
Based on the type of missing elements, there are three sub-tasks: head reasoning to infer $(?, r_q, t)$, tail reasoning to infer $(h, r_q, ?)$, and relation reasoning to infer $(h, ?, t)$. 
This paper mainly focuses on the tail reasoning task, as the other tasks can be transformed into the same form.

%% file: sections/4.method.tex
In this section, we introduce the construction of the \model.
Specifically, we focus on how to create the proposed attention mechanism and integrate it into a transformer model.

\subsection{Attention Computation of \model} \label{sec:sec.4.1}
We adopt a modified definition of self-attention, similar to the formulation used in \citet{TsaiBYMS19} and \citet{ChenOB22}.
Assuming that $\boldsymbol{X} \in \mathbb{R}^{n\times d}$ is input features, we have:
\begin{equation} \label{eq:eq.2}
    \text{Attn}(\boldsymbol{x}_u) = \sum_{v \in \mathcal{V}}  \frac{\kappa(f_q(\boldsymbol{x}_u),f_q(\boldsymbol{x}_v))}{\sum_{w \in \mathcal{V}} \kappa(f_q(\boldsymbol{x}_u),f_q(\boldsymbol{x}_w))} f_v(\boldsymbol{x}_v),
\end{equation}
here, $f_q(\cdot)$ represents the query function, $f_v(\cdot)$ represents the value function, and $\kappa(\cdot,\cdot): \mathbb{R}^{d} \times \mathbb{R}^{d}  \rightarrow \mathbb{R}^+$ is a positive-definite kernel that measures the pairwise similarity. 

In a knowledge graph, different relation types $r \in \mathcal{R}$ determine various modes of interaction between entity pairs. 
Consequently, a straightforward approach is to incorporate relation-specific attention to model these diverse interactions:
\begin{equation} \label{eq:eq.3}
    \text{Attn}(\boldsymbol{x}_u)\!=\!\sum_{r\in \mathcal{R}}\! \sum_{u \in \mathcal{V}}  \frac{\kappa(f_{q}^r(\boldsymbol{x}_u),f_{q}^r(\boldsymbol{x}_v))}{\sum_{w \in \mathcal{V}} \kappa(f_{q}^r(\boldsymbol{x}_u),f_{q}^r(\boldsymbol{x}_w))} f_{v}^r(\boldsymbol{x}_v),
\end{equation}
where $f^r(.)$ means a relation-specific function.
However, this approach will result in significant computational overhead~($\mathcal{O}(\lvert\mathcal{R} \rvert \cdot \lvert \mathcal{V} \rvert^2)$) and suboptimal optimization due to a high degree of freedom.
Inspired by instance-based learning~\cite{StanfillW86,CuiC22}, we consider modeling the pairwise interactions based on the plausibility of entity pairs as query prototypes.
\begin{definition}[Query Prototype] \label{def.1}
Given a relation $r\in \mathcal{R}$, entity $u\in \mathcal{V}$ and $v\in \mathcal{V}$ are the prototypes for relation $r$ if $(u,r,?) \wedge(v,r,?) \ne \emptyset$.
\end{definition}
For instance, \texttt{Barack Obama} and \texttt{Joseph Biden} are the prototypes for relation \texttt{Nationality}, since we have $(\texttt{Barack Obama}, \texttt{Nationality}, ?) \wedge (\texttt{Joseph Biden}, \texttt{Nationality}, ?) = \{\texttt{USA}\}$.
Based on \cref{def.1}, given query $(h, r_q, ?)$, we propose to calculate attention scores between entity pairs referring to their plausibility as prototypes. 
Our motivation is that the greater the plausibility of two entities as prototypes, the more \textit{similarity} they exhibit under the current query relation, leading to higher attention scores.
In this way, for each query in the dataset, we do not need to consider the interaction of entities for all relations, reducing the complexity to $\mathcal{O}(|\mathcal{V}|^2)$.
Our idea considers the plausibility of two entities $u$ and $v$ being prototypes to the query relation $r_q$ and enables information propagation between them based on the plausibility.
To achieve this, we propose to utilize the function $f_q(\cdot)$ to obtain informative representations of entities to distinguish query prototypes, the function $\kappa(\cdot,\cdot)$ to serve as the score function for measuring plausibility, and lastly, $f_v(\cdot)$ to generate propagated information.
Next, we will provide a detailed explanation of constructing the attention mechanism by defining the functions $f_q(\cdot)$, $f_v(\cdot)$, and $\kappa(\cdot,\cdot)$.

\paragraph{Query Function.}
Query function $f_q(\cdot)$ is designed to provide informative representations, denoted as $\boldsymbol{\widetilde{z}}_u$, for each entity $u$ to distinguish query prototypes effectively.
A key insight is that considering the context of each entity on the knowledge graph (\textit{i.e.}, its $k$-hop neighbor facts) is essential in determining query prototypes. 
For example, entities such as \texttt{Barack Obama} and \texttt{Joseph Biden} can serve as prototypes for the relation \texttt{Nationality} due to their common neighbor relation \texttt{President of}. 
We introduce Q-RMPNN, a relational message-passing neural network~\cite{GilmerSRVD17,XuHLJ19,0004RL21} designed to incorporate neighbor facts into the entity representations.
Q-RMPNN consists of two stages: (1) each entity sends relational messages to its neighbors; (2) each entity aggregates the received relational messages and updates its representation. 
Drawing inspiration from knowledge graph embeddings, we generate a relational message $\boldsymbol{m}_{v|u,r}$ for each fact $r(u,v)$ by maximizing its continuous truth value~\cite{WangSWS23}:
\begin{equation}
    \boldsymbol{m}_{v|u,r} = \mathop{\arg \max}_{\boldsymbol{w} \in \mathcal{D}} \phi (\boldsymbol{z}_u, \boldsymbol{\hat{r}}, \boldsymbol{w}) = g(\boldsymbol{z}_u, \boldsymbol{\hat{r}}),
\end{equation}
where $\mathcal{D}$ indicates the range domain for the entity embeddings and $\boldsymbol{z}_u$ is the representation of entity $u$ within $f_q(\cdot)$, $\boldsymbol{\hat{r}}$ is the representation of relation $r$ conditioned on query relation $r_q$, $\phi(\cdot,\cdot,\cdot)$ is a score function and $g(\cdot,\cdot)$ is an estimation function.
In this paper, we adopt the DistMult method~\cite{YangYHGD14a} as the foundation for our implementation similar to~\citet{ZhuZXT21}, to say:
\begin{equation}
    \begin{aligned}
        & \boldsymbol{\hat{r}} = \boldsymbol{R}[r_q] \cdot \boldsymbol{W}_r + \boldsymbol{b}_r, \; g(\boldsymbol{z}_u, \boldsymbol{\hat{r}}) = \boldsymbol{z}_u \odot \boldsymbol{\hat{r}}, \\
        & \phi(\boldsymbol{z}_u, \boldsymbol{\hat{r}}, \boldsymbol{w}) = \left\langle g(\boldsymbol{z}_u, \boldsymbol{\hat{r}}),\boldsymbol{w} \right\rangle,
    \end{aligned}
\end{equation}
where $\boldsymbol{R}$ is the input features of relation set $\mathcal{R}$, and $\boldsymbol{W}_r,\boldsymbol{b}_r$ is the relation-specific parameters.
At the $l$-th layer of $f_q(\cdot)$, we fuse $\boldsymbol{z}_u^{(l)}$ for each $u \in \mathcal{V}$ by the summation of aggregated information from the $(l-1)$-th layer:
\begin{equation} \label{eq:eq.6}
    \begin{aligned}
        \boldsymbol{z}_u^{(0)} &\leftarrow [\boldsymbol{x}_u, \boldsymbol{\epsilon}], \\
        \boldsymbol{z}_u^{(l)} &\leftarrow \Phi^{(l)} \left(\boldsymbol{\alpha}^{(l)} \cdot \boldsymbol{z}_u^{(l-1)} + \sum_{r(v,u)\in\mathcal{E}} \boldsymbol{m}_{u|v,r}^{(l-1)}  \right),
    \end{aligned}
\end{equation}
where $\boldsymbol{x}_u \in \mathbb{R}^d$ represents the input features, $\boldsymbol{\epsilon} \sim N(\boldsymbol{0}, \boldsymbol{I})$ denotes Gaussian noise to distinguish different source entities, $[\cdot,\cdot]$ indicates concatenate function, $\boldsymbol{\alpha}^{(l)}$ captures layer-specific parameters to retain the original information, and $\boldsymbol{\Phi}^{(l)}(\cdot)$ represents a layer-specific update function parameterized by an MLP network. 
After $\widetilde{L}$ layers of updates, we obtain the final representation $\boldsymbol{\widetilde{z}}_u$ given by $\boldsymbol{\widetilde{z}}_u = \boldsymbol{z}_u^{(\widetilde{L})}$.
Through Q-RMPNN, $\boldsymbol{\widetilde{z}}_u$ is optimized to capture $\widetilde{L}$-hop neighbor facts, resulting in a structure-aware and knowledge-oriented representation.

\paragraph{Value Function.}
We leverage the value function to generate pairwise representations $\boldsymbol{\hat{z}}_u$ conditioned on the query relation $r_q$ for the query $(h,r_q,?)$. 
The structural information between node pairs is crucial in knowledge graph reasoning, which can be viewed as a link prediction task~\cite{ZhangLXWJ21}. 
However, it is challenging for the attention mechanism to capture the structural information of the input graph explicitly. 
To address this, we design V-RMPNN to encode pairwise structural information into the value. 
Specifically, V-RMPNN is implemented as follows:
\begin{equation} \label{eq:eq.7}
    \begin{aligned}
        \boldsymbol{z}_{u|h,r_q}^{(0)} &\leftarrow [\boldsymbol{x}_u, \mathbb{I}_{u=h} \odot \boldsymbol{1}], \\
        \boldsymbol{z}_{u|h,r_q}^{(l)} &\leftarrow \Psi^{(l)} \left(\boldsymbol{\beta}^{(l)} \cdot \boldsymbol{z}_{u|h,r_q}^{(l-1)} + \sum_{r(v,u)\in\mathcal{E}} \boldsymbol{m}_{u|v,r}^{(l-1)}  \right),
    \end{aligned}
\end{equation}
where $\boldsymbol{x}_u$, $\boldsymbol{m}_{u|v,r}^{(l-1)}$, $\boldsymbol{\beta}^{(l)}$, $[\cdot,\cdot]$ and $\boldsymbol{\Psi}^{(l)}(\cdot)$ are defined similarly to Q-RMPNN. 
The term $\mathbb{I}_{u=h} \odot \boldsymbol{1}$ represents head entity labeling features to enhance node representation, which is essential for the link prediction task~\cite{ZhangLXWJ21,YouGYL21}. 
After $\widehat{L}$ layers of updates, the final representation $\boldsymbol{\hat{z}}_u$ is obtained as $\boldsymbol{\hat{z}}_u = \boldsymbol{z}_{u|h,r_q}^{(\widehat{L})}$.

\begin{figure}[tb]
    \centering
    \includegraphics[scale=0.6]{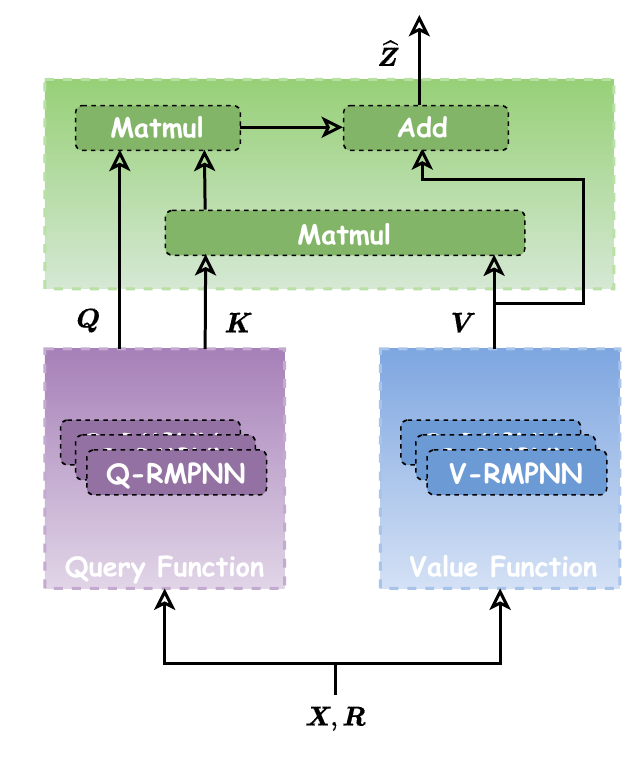}
    \caption{Overview of the proposed attention mechanism, which takes entity features $\boldsymbol{X}$ and $\boldsymbol{R}$ as input and outputs $\boldsymbol{\widehat{Z}}$ for all $u\in\mathcal{V}$.}
    \label{fig:fig.1}
\end{figure}

\paragraph{Kernel Function.}
The kernel function $\kappa(\boldsymbol{\widetilde{z}}_u,\boldsymbol{\widetilde{z}}_v)$ is employed to quantify pairwise similarity.
One common choice is the exponential kernel specified as follows:
\begin{equation}
    \kappa_{\text{exp}}(\boldsymbol{\widetilde{z}}_u,\boldsymbol{\widetilde{z}}_v) = \exp \left( \langle \boldsymbol{\widetilde{z}}_u \boldsymbol{W}_1,\boldsymbol{\widetilde{z}}_v \boldsymbol{W}_2 \rangle / \sqrt{d} \right),
\end{equation}
where $\boldsymbol{W}_1$ and $\boldsymbol{W}_2$ are the linear projection matrixes and the bias is omitted for convenience.
However, the \textit{dot-then-exponentiate operation} will lead to quadratic complexity since we must calculate the inner product of the query matrix with the key matrix before the exponentiate function.
In practice, we approximate the exponential function $\exp(\cdot)$ using the \textit{first-order Taylor expansion} around zero~\cite{abs-2007-14902,WuYZHWY23,DassWSLYWL23}, resulting in:
\begin{equation} \label{eq.9}
    \kappa_{\text{exp}}(\boldsymbol{\widetilde{z}}_u,\boldsymbol{\widetilde{z}}_v) \approx 1 +  \langle \boldsymbol{\widetilde{z}}_u\boldsymbol{W}_1,\boldsymbol{\widetilde{z}}_v\boldsymbol{W}_2 \rangle.
\end{equation}
To ensure non-negativity of the attention scores, we further normalize the inputs by their Frobenius norm, leading to the kernel function implementation as follows:
\begin{equation} \label{eq.10}
    \kappa(\boldsymbol{\widetilde{z}}_u,\boldsymbol{\widetilde{z}}_v) = 1 +  \left\langle \frac{\boldsymbol{\widetilde{z}}_u\boldsymbol{W}_1}{\Vert \boldsymbol{\widetilde{z}}_u\boldsymbol{W}_1 \Vert_{\mathcal{F}}}, \frac{\boldsymbol{\widetilde{z}}_v\boldsymbol{W}_2}{\Vert \boldsymbol{\widetilde{z}}_v\boldsymbol{W}_2 \Vert_{\mathcal{F}} } \right\rangle.
\end{equation}
Based on \cref{eq.9} and \cref{eq.10}, we can initially compute the inner product of the key matrix with the value matrix, followed by calculating the inner product of the query matrix with the resultant matrix as shown in \cref{fig:fig.1}, thereby reducing the complexity to linear.
\begin{restatable}[Approximation Error for Exponential Kernel]{theorem}{theoa} 
\label{theo:theo.1} 
For each $u,v \in \mathcal{V}$, the approximation error $\Delta = \lvert \kappa(\boldsymbol{\widetilde{z}}_u,\boldsymbol{\widetilde{z}}_v) - \kappa_{\text{exp}}(\boldsymbol{\widetilde{z}}_u,\boldsymbol{\widetilde{z}}_v) \rvert$ will be bounded by $\mathcal{O}(e^{\gamma}/2)$, where $\gamma \in (0, 1)$.
\end{restatable}
We provide the proof in \cref{app:app.theo1}.
\cref{theo:theo.1} demonstrates that the approximation error remains bounded regardless of the size of KGs, which implies \cref{eq.9} is a stable estimation.

\paragraph{Summary.} 
Now, we present the detailed computational process of the proposed attention mechanism, as illustrated in \cref{fig:fig.1}.
Assume $\boldsymbol{X} \in \mathbb{R}^{\lvert \mathcal{V} \rvert \times d}$ is the input feature matrix of entities, $\boldsymbol{R} \in \mathbb{R}^{\lvert \mathcal{R} \rvert \times d}$ is the input features matrix of relations, $f_q(\cdot)$ and $f_v(\cdot)$ are the query and value function, and $\kappa(\cdot,\cdot)$ is the kernel function.
We firstly obtain $\boldsymbol{\widetilde{Z}} = [f_q(\boldsymbol{x}_u)]_{u\in\mathcal{V}}$ and $\boldsymbol{\widehat{Z}} = [f_v(\boldsymbol{x}_u)]_{u\in\mathcal{V}}$ by iteratively applying \cref{eq:eq.6} and \cref{eq:eq.7}.
Then the output feature matrix $\boldsymbol{\overline{Z}}$ can be calculated as follows:
\begin{equation} \label{eq:eq.10}
    \begin{aligned}
        \boldsymbol{Q} = &\frac{\boldsymbol{\widetilde{Z}}\boldsymbol{W}_1}{\Vert \boldsymbol{\widetilde{Z}}\boldsymbol{W}_1 \Vert_{\mathcal{F}}}, \;\;\;\;\;
        \boldsymbol{K} = \frac{\boldsymbol{\widetilde{Z}}\boldsymbol{W}_2}{\Vert \boldsymbol{\widetilde{Z}}\boldsymbol{W}_2 \Vert_{\mathcal{F}}}, 
        \;\;\;\;\; \boldsymbol{V} = \boldsymbol{\widehat{Z}},  \\
        & \boldsymbol{D} = \text{diag} \left(\boldsymbol{1} + \frac{\boldsymbol{Q} \left(\boldsymbol{K}^T \boldsymbol{1} \right) + \lvert \mathcal{V} \rvert}{\lvert \mathcal{V} \rvert}\right), \\
        &\boldsymbol{\overline{Z}} = \boldsymbol{D}^{-1} \left[\boldsymbol{V} + \frac{\boldsymbol{1}^T\boldsymbol{V} + \boldsymbol{Q} \left(\boldsymbol{K}^T \boldsymbol{V} \right)}{\lvert \mathcal{V} \rvert} \right],
    \end{aligned}
\end{equation}
where we omit the bias and $\text{diag}(\cdot)$ means a transformation to diagonal matrix. 
We further show the overall procedure in \cref{alg:alg.1}.
We can observe that the time complexity of \cref{eq:eq.10} is linearly related to $ \lvert \mathcal{V} \rvert $ and $\lvert \mathcal{E} \rvert$, which is scalable to massive KGs.
We will discuss the details of complexity in \cref{sec:sec.4.3}.

\begin{algorithm}[!tb]
   \caption{Attention Computation}
   \label{alg:alg.1}
    \begin{algorithmic}[1]
        \INPUT knowledge graph $\mathcal{G} = \{\mathcal{V}, \mathcal{E}, \mathcal{R}\}$, query $(h,r_q,?)$, entity features $\boldsymbol{X} = [\boldsymbol{x}_u]_{u\in\mathcal{V}}$, relation features $\boldsymbol{R} = [\boldsymbol{r}_p]_{p \in \mathbf{R}}$, and model parameters include $\boldsymbol{\widetilde{W}}, \boldsymbol{\tilde{b}}, \Phi, \boldsymbol{\alpha}, \boldsymbol{\widehat{W}}, \boldsymbol{\hat{b}}, \Psi, \boldsymbol{\beta}, \boldsymbol{W}_1,$ and $\boldsymbol{W}_2$;
        
        \OUTPUT output features $\boldsymbol{\overline{Z}}$;

        \FOR{$l=0$ {\bfseries to} $\widetilde{L}$} 
            \IF{$l = 0$}
                \STATE $\boldsymbol{\tilde{z}}_u^{(0)} \leftarrow [\boldsymbol{x}_u, \epsilon]$;
            \ELSE
                \STATE $\boldsymbol{\tilde{h}}_v \leftarrow \boldsymbol{\tilde{z}}_v^{(l-1)} \odot \left(\boldsymbol{r}_q \cdot \boldsymbol{\widetilde{W}}_{r} + \tilde{b}_r\right)$;
                \STATE $\boldsymbol{\tilde{m}}_{u|v,r}^{(l-1)} \leftarrow \mathop{\arg \max}_{\boldsymbol{w} \in \mathcal{D}} \left\langle \boldsymbol{\tilde{h}}_v, \boldsymbol{w} \right\rangle$;
                \STATE $\boldsymbol{\tilde{z}}_u^{(l)} \leftarrow \Phi^{(l)} \left(\boldsymbol{\alpha}^{(l)} \cdot \boldsymbol{\tilde{z}}_u^{(l-1)} + \sum_{r(v,u)\in\mathcal{E}} \boldsymbol{\tilde{m}}_{u|v,r}^{(l-1)}  \right)$;
            \ENDIF
            \STATE \algorithmiccomment{\textit{\color[HTML]{0C37D8}{Iteratively compute $\boldsymbol{\widetilde{Z}}$ according to \cref{eq:eq.6} taking $\boldsymbol{X}$ and $\boldsymbol{r}_q$ as input}}};
        \ENDFOR

        \FOR{$l=0$ {\bfseries to} $\widehat{L}$}
            \IF{$l = 0$}
                \STATE $\boldsymbol{\hat{z}}_{u}^{(0)} \leftarrow [\boldsymbol{x}_u, \mathbb{I}_{u=h} \odot \boldsymbol{1}]$;
            \ELSE
                \STATE $\boldsymbol{\hat{h}}_v \leftarrow \boldsymbol{\hat{z}}_v^{(l-1)} \odot \left(\boldsymbol{r}_q \cdot \boldsymbol{\widehat{W}}_{r} + \hat{b}_r\right)$;
                \STATE $\boldsymbol{\hat{m}}_{u|v,r}^{(l-1)} \leftarrow \mathop{\arg \max}_{\boldsymbol{w} \in \mathcal{D}} \left\langle \boldsymbol{\hat{h}}_v, \boldsymbol{w} \right\rangle$;
                \STATE $\boldsymbol{\hat{z}}_{u}^{(l)} \leftarrow \Psi^{(l)} \left(\boldsymbol{\beta}^{(l)} \cdot \boldsymbol{\hat{z}}_{u}^{(l-1)} + \sum_{r(v,u)\in\mathcal{E}} \boldsymbol{\hat{m}}_{u|v,r}^{(l-1)}  \right)$;
            \ENDIF
            \STATE \algorithmiccomment{\textit{\color[HTML]{0C37D8}{Iteratively compute $\boldsymbol{\widehat{Z}}$ according to \cref{eq:eq.7} taking $\boldsymbol{X}$, $\boldsymbol{r}_q$, and $h$ as input}}};
        \ENDFOR

        \STATE Obtain $\boldsymbol{Q}$, $\boldsymbol{K}$ and $\boldsymbol{V}$ by $\boldsymbol{Q} \leftarrow \frac{\boldsymbol{\widetilde{Z}}\boldsymbol{W}_1}{\Vert \boldsymbol{\widetilde{Z}}\boldsymbol{W}_1 \Vert_{\mathcal{F}}}$, $\boldsymbol{K} \leftarrow \frac{\boldsymbol{\widetilde{Z}}\boldsymbol{W}_2}{\Vert \boldsymbol{\widetilde{Z}}\boldsymbol{W}_2 \Vert_{\mathcal{F}}}$ and $\boldsymbol{V} \leftarrow \boldsymbol{\widehat{Z}}$;

        \STATE $\boldsymbol{D} \leftarrow \text{diag} \left(\boldsymbol{1} + \frac{\boldsymbol{Q} \left(\boldsymbol{K}^T \boldsymbol{1} \right) + \lvert \mathcal{V} \rvert}{\lvert \mathcal{V} \rvert}\right)$;
        
        \STATE $\boldsymbol{\overline{Z}} \leftarrow \boldsymbol{D}^{-1} \left[\boldsymbol{V} + \frac{\boldsymbol{1}^T\boldsymbol{V} + \boldsymbol{Q} \left(\boldsymbol{K}^T \boldsymbol{V} \right)}{\lvert \mathcal{V} \rvert} \right]$;
        \algorithmiccomment{\textit{\color[HTML]{0C37D8}{kernel function computation}}}
    \end{algorithmic}
\end{algorithm}

\subsection{Overall Architecture of \model}
Based on the attention definition provided in \cref{sec:sec.4.1}, we complete the construction of our proposed model \model.
The remaining components of \model adhere to the standard transformer architecture described in \citet{VaswaniSPUJGKP17}.
Each layer of \model comprises an attention function and a feedforward network. 
In addition, a skip-connection is applied after the attention function, and normalization layers are inserted before and after the feedforward network.
Finally, we stack $L$ layers to construct \model. 

At the start of training, we initialize the entity feature matrix $\boldsymbol{X}$ and relation feature matrix $\boldsymbol{R}$. 
Specifically, $\boldsymbol{X}$ is set to an all-zero vector, while $\boldsymbol{R}$ is randomly initialized and is learnable. 
We then iteratively calculate the entity representations as follows:
\begin{equation}
    \begin{aligned}
        \boldsymbol{A}^{(l)} &= \text{LayerNorm}^{(l)}_1 \left(\boldsymbol{X}^{(l - 1)} + \text{Attn}^{(l)}\left(\boldsymbol{X}^{(l - 1)}, \boldsymbol{R}\right) \right), \\
        \boldsymbol{X}^{(l)} &= \text{LayerNorm}^{(l)}_2 \left( \boldsymbol{A}^{(l)} + \text{FFN}^{(l)}(\boldsymbol{A}^{(l)}) \right). \\
    \end{aligned}
\end{equation}
Lastly, we obtain the final representation of entities $\boldsymbol{X}^{(L)}$.
We update model parameters by optimizing a negative sampling loss~\cite{MikolovSCCD13,SunDNT19} using Adam optimizer~\cite{KingmaB14}.
The loss function is defined for each training fact $(h,r,t)$ as:
\begin{equation}
    \begin{split}
        \mathcal{L} = -\log(\sigma(t|h,r)) 
         -  \sum_{t^\prime} \log \left(  1-\sigma(t^\prime|h,r) \right) ,
    \end{split}
\end{equation}
where $\sigma(\cdot|h,r)$ represents the score of a candidate fact, computed by a multilayer perceptron (MLP) and a sigmoid function using the output feature $\boldsymbol{X}^{(L)}$ and $t^\prime  \in \mathcal{V}\setminus \{t\}$ indicates the negative samples.

\subsection{Discussion} \label{sec:sec.4.3}
In this section, we provide a thorough discussion of the \model, focusing on the analysis of its time complexity and expressivity.

\input{tables/transductive_main_results}

\paragraph{Time Complexity.}
Time complexity of our proposed model primarily depends on the attention function. 
We can break down the time complexity into the query function, the value function, and the kernel function. 
The query function is called $L$ times, with each call taking $\mathcal{O}(\widetilde{L}(\lvert \mathcal{E} \rvert d + \lvert \mathcal{V} \rvert d^2))$. 
Similarly, the value function is called $L$ times, with a single call taking $\mathcal{O}(\widehat{L}(\lvert \mathcal{E} \rvert d + \lvert \mathcal{V} \rvert d^2))$. 
Finally, the kernel function is called $L$ times, with each call taking $\mathcal{O}(\lvert \mathcal{V} \rvert d^2)$. 
Therefore, the total complexity amounts to $\mathcal{O}(L(\widetilde{L} + \widehat{L})(\lvert \mathcal{E} \rvert d + \lvert \mathcal{V} \rvert d^2))$. 
The above conclusion shows that the time complexity of our proposed method is linearly related to the number of facts and entities in the knowledge graph, exhibiting better scalability.

\paragraph{Expressivity Analysis.}
\citet{huang2023theory} introduced a variant of the Weisfeiler-Leman Test~\cite{weisfeiler1968reduction,XuHLJ19,Barcelo00O22} called the \textit{Relational Asymmetric Local 2-WL Test} (\texttt{RA-WL}$_2$) to evaluate the expressive power of message-passing networks in the knowledge graph reasoning task. 
We formally demonstrate the expressivity of \model based on \texttt{RA-WL}$_2$. 
This is stated in the following theorem:
\begin{restatable}[Expressivity of \model]{theorem}{theoe} 
\label{theo:theo.2}
    Assuming the estimated graph by the attention layer of \model is $\mathcal{\tilde{E}}$, the attention layer of \model can achieve at least the same level of expressive ability as \texttt{RA-WL}$_2$ on the extended graph $\mathcal{E} \cup \mathcal{\tilde{E}}$.
\end{restatable}
We provide proof in \cref{app:app.theo2}. 
The estimated graph $\mathcal{\tilde{E}}$ can be viewed as a collection of facts of special relations determining the equivalence between entity pairs for a specific query.
\model exhibits stronger expressive power due to its ability to propagate information on an extended graph, distinguishing it from vanilla path-based methods.

%% file: tables/transductive_main_results.tex
\begin{table*}[htb]
    \centering
    \begin{adjustbox}{width=\textwidth}
    
    \begin{tabular}{lcccccccccccc}
    \toprule
    
    \multicolumn{1}{c}{\multirow{2}[4]{*}{\bf{Method}}} & \multicolumn{3}{c}{\bf{FB15k-237}} & \multicolumn{3}{c}{\bf{WN18RR}} & \multicolumn{3}{c}{\bf{NELL-995}} & \multicolumn{3}{c}{\bf{YAGO3-10}}   \\
    
    \cmidrule(r){2-4} \cmidrule(r){5-7} \cmidrule(r){8-10} \cmidrule(r){11-13} 
    & \multicolumn{1}{c}{MRR} & \multicolumn{1}{c}{H@1} & \multicolumn{1}{c}{H@10} & \multicolumn{1}{c}{MRR} & \multicolumn{1}{c}{H@1} & \multicolumn{1}{c}{H@10} & \multicolumn{1}{c}{MRR} & \multicolumn{1}{c}{H@1} & \multicolumn{1}{c}{H@10} & \multicolumn{1}{c}{MRR} & \multicolumn{1}{c}{H@1} & \multicolumn{1}{c}{H@10}   \\

    \midrule
    
    TransE{\small~\cite{BordesUGWY13}} & 0.330 & 23.2 & 52.6 & 0.222 & 1.4 & 52.8 & 0.507 & 42.4 & 64.8 & 0.510 & 41.3 & 68.1 \\
    DistMult{\small~\cite{YangYHGD14a}}  & 0.358 & 26.4 & 55.0 & 0.455 & 41.0 & 54.4 & 0.510 & 43.8 & 63.6 & 0.566 & 49.1 & 70.4 \\
    RotatE{\small~\cite{SunDNT19}}  & 0.337 & 24.1 & 53.3 & 0.477 & 42.8 & 57.1 & 0.508 & 44.8 & 60.8 & 0.495 & 40.2 & 67.0 \\
    HousE{\small~\cite{Li0LH0LSWDSXZ22}}  & 0.361 & 26.6 & 55.1 & 0.511 & 46.5 & 60.2 & 0.519 & 45.8 & 61.8 & 0.571 & 49.1 & \underline{71.4} \\

    \hdashline

    DRUM{\small~\cite{SadeghianADW19}}  & 0.343 & 25.5 & 51.6 & 0.486 & 42.5 & 58.6 & 0.532 & 46.0 & \underline{66.2} & 0.531 & 45.3 & 67.6   \\
    CompGCN{\small~\cite{VashishthSNT20}} & 0.355 & 26.4 & 53.5 & 0.479 & 44.3 & 54.6 & 0.463 & 38.3 & 59.6 & 0.421 & 39.2 & 57.7 \\
    RNNLogic{\small~\cite{QuCXBT21}}  & 0.344 & 25.2 & 53.0 & 0.483 & 44.6 & 55.8 & 0.516 & 46.3 & 57.8 & 0.554 & 50.9 & 62.2 \\
    NBFNet{\small~\cite{ZhuZXT21}}  & 0.415 & 32.1 & \underline{59.9} & 0.551 & 49.7 & 66.6 & 0.525 & 45.1 & 63.9 & 0.563 & 48.0 & 70.8 \\  
    RED-GNN{\small~\cite{ZhangY22}}  & 0.374 & 28.3 & 55.8 & 0.533 & 48.5 & 62.4 & 0.543 & 47.6 & 65.1 & 0.556 & 48.3 & 68.9 \\
    A*Net{\small~\cite{zhu2023net}}  & 0.411 & 32.1 & 58.6 & 0.549 & 49.5 & 65.9 & 0.521 & 44.7 & 63.1 & 0.556 & 47.0 & 70.7 \\
    AdaProp{\small~\cite{ZhangZY0023}}  & \underline{0.417} & \underline{33.1} & 58.5 & 0.562 & 49.9 & 67.1 & \underline{0.554} & \underline{49.3} & 65.5 & \underline{0.573} & \underline{51.0} & 68.5 \\
    \textsc{Ultra}{\small~\cite{abs-2310-04562}}  & 0.368 & - & 56.4 & 0.480 & - & 61.4 & 0.509 & - & 66.0 & 0.557 & - & 71.0 \\

    \hdashline

    HittER{\small~\cite{ChenLG0ZJ21}} & 0.373 & 27.9 & 55.8 & 0.503 & 46.2 & 58.4 & - & - & - & - & - & - \\
    SimKGC{\small~\cite{0046ZWL22}} & 0.336 & 24.9 & 51.1 & \bf{0.666} & \bf{58.5} & \bf{80.0} & 0.501 & 42.6 & 65.3 & 0.211 & 14.1 & 35.1 \\
    KGT5{\small~\cite{SaxenaKG22}} & 0.276 & 21.0 & 41.4 & 0.508 & 48.7 & 54.4 & - & - & -  & 0.426 & 36.8 & 52.8 \\
    
    \hdashline

    \rowcolor{Tan!20}
    \bf{\model} & \bf{0.430} & \bf{34.3} & \bf{60.8} & \underline{0.579} & \underline{52.8} & \underline{68.7} & \bf{0.566} & \bf{50.2} & \bf{67.5} & \bf{0.615} & \bf{54.7} & \bf{73.4} \\
    
    \bottomrule
    \end{tabular}
    \end{adjustbox}
    
  \caption{Transductive knowledge graph reasoning performance for 4 different datasets. The best results are \textbf{boldfaced} and the second-best results are \underline{underlined}. Our proposed model, \model, achieves SOTA performance in most cases marked by {\colorbox{Tan!20}{\textbf{tan}}}.}
  \label{tab:tab.1}
\end{table*}

%% file: sections/5.experiment.tex
In this section, we conduct empirical studies motivated by the following aspects:
\begin{itemize}[itemsep=3pt,topsep=0pt,parsep=1pt]
    \item \textbf{Transductive Performance.} As a general knowledge graph reasoning task, we aim to demonstrate the performance of \model in transductive knowledge graph reasoning tasks.
    \item \textbf{Inductive Performance.} Our proposed model supports reasoning in inductive settings. How does \model perform in inductive knowledge graph reasoning tasks?
    \item \textbf{Ablation Study.} How important are the various components within our proposed framework? For instance, how does model performance change if we omit the query function?
    \item \textbf{Further Experiments.} We conduct additional experiments to further showcase the effectiveness of \model. For example, does \model perform better for longer paths?
\end{itemize}

Our code is available at \url{https://github.com/jnanliu/KnowFormer}.

\input{tables/inductive_main_results}

\subsection{Transductive Performance} 
\paragraph{Datasets.} 
We conduct experiments on four widely utilized transductive knowledge graph reasoning datasets: FB15k-237~\cite{ToutanovaC15}, WN18RR~\cite{DettmersMS018}, NELL-995~\cite{XiongHW17} and YAGO3-10~\cite{MahdisoltaniBS15}.
For consistency, we utilize the same data splits as in prior studies~\cite{ZhuZXT21,ZhangZY0023}.

\vspace{-0.2cm}
\paragraph{Baselines.}
We compare \model to several prominent baselines, categorized into three classes: 
\begin{itemize}[itemsep=3pt,topsep=0pt,parsep=1pt]
\item \textit{Embedding-based methods}, including TransE~\cite{BordesUGWY13}, DistMult~\cite{YangYHGD14a}, RotatE~\cite{SunDNT19}, and HousE~\cite{Li0LH0LSWDSXZ22}. These methods learn embeddings for entities and relations and perform reasoning based on distance or similarity. 
\item \textit{Path-based methods}, including DRUM~\cite{SadeghianADW19}, CompGCN~\cite{VashishthSNT20}, RNNLogic~\cite{QuCXBT21}, NBFNet~\cite{ZhuZXT21}, RED-GNN~\cite{ZhangY22}, A*Net~\cite{zhu2023net}, AdaProp~\cite{ZhangZY0023}, and \textsc{Ultra}~\cite{abs-2310-04562}. These methods conduct reasoning by utilizing the path information connecting entities. 
\item \textit{Text-based methods}, including HittER~\cite{ChenLG0ZJ21}, SimKGC~\cite{0046ZWL22}, and KGT5~\cite{SaxenaKG22}. These methods utilize textual information for reasoning in knowledge graphs.
\end{itemize}

\paragraph{Results and Analysis.}
We evaluate the performance of reasoning using standard metrics~\cite{WangMWG17,JiPCMY22}, namely MRR ($\uparrow$), Hit@1 ($\uparrow$), and Hit@10 ($\uparrow$). 
The results on various datasets are presented in \cref{tab:tab.1}.
Across all metrics, \model demonstrates a substantial performance advantage over the best baseline method in FB15k-237, NELL-995, and YAGO3-10. 
Particularly in the large-scale YAGO3-10 dataset, \model exhibits a significant performance advantage over the best baseline methods. 
This highlights the high effectiveness of \model in transductive knowledge graph reasoning.
In the WN18RR dataset, \model achieves the second-best performance among all baselines.
Notably, SimKGC~\cite{0046ZWL22} performs exceptionally well in this dataset, which can be attributed to its ability to acquire knowledge from extensive pretrained data, allowing it to capture the semantic relationships in WN18RR as a subset of a comprehensive English lexical database.
Conversely, methods like SimKGC that rely on pretrained language models tend to underperform in datasets that require domain-specific knowledge. 
In contrast, our proposed text-free method relies solely on the structure of the knowledge graphs, resulting in enhanced robustness and scalability.

\subsection{Inductive Performance} 
Inductive reasoning has become a prominent research topic~\cite{HamiltonYL17,TeruDH20} due to the ubiquitous occurrence of emergent entities in real-world applications~\cite{ZhangC20}.
In this part, we will show the performance of \model on the inductive knowledge graph reasoning task.
Note that in this paper, we focus on the \textit{semi-inductive} task, which has an invariant relation set.
However, we believe that our method can be adapted to handle full-inductive tasks with slight modifications. 
We leave this aspect as future work to address.

\paragraph{Datasets.} 
In line with \citet{TeruDH20}, we employ the same data divisions of FB15k-237~\cite{ToutanovaC15}, WN18RR~\cite{DettmersMS018}, and NELL-995~\cite{XiongHW17}.
Each of these divisions comprises 4 versions, resulting in a total of 12 subsets. 
It is worth noting that each subset has a unique separation between the training and test sets.
Specifically, the training and test sets within each subset have unique sets of entities while sharing the same set of relations.

\paragraph{Baselines.}
We compare \model with 7 baseline methods for inductive knowledge graph reasoning, including DRUM~\cite{SadeghianADW19}, NBFNet~\cite{ZhuZXT21}, RED-GNN~\cite{ZhangY22}, A*Net~\cite{zhu2023net}, AdaProp~\cite{ZhangZY0023}, SimKGC~\cite{0046ZWL22}, and \textsc{InGram}~\cite{LeeCW23}.
Note that \textsc{Ultra}~\cite{abs-2310-04562} is not included as a baseline model, since we consider \textsc{Ultra} to be a distinctive method based on the pre-training and fine-tuning paradigm, which can benefit from a large-scale amount of training data.
However, our method still achieves comparable results to \textsc{Ultra} on some datasets. 
Please refer to the experimental results in \cref{app:app.aer.ultra} for more details.

\paragraph{Results and Analysis.}
We also evaluate the performance by standard metrics~\cite{WangMWG17,JiPCMY22}, including MRR ($\uparrow$), Hit@1 ($\uparrow$), and Hit@10 ($\uparrow$). 
\cref{tab:tab.2} showcases the inductive performance of \model and baselines.
We have observed that \model consistently achieves the highest performance on the majority of metrics across all versions of inductive datasets. 
Furthermore, our method also produces competitive results for the remaining metrics. 
Compared to transductive settings, \model demonstrates a relatively small performance gap. 
This can be attributed to the limited size of inductive datasets. 
Path-based methods can attain superior performance by utilizing shorter paths and avoiding excessive compression of information.
However, \model maintains a performance advantage, indicating its effectiveness. 
Another noteworthy observation is the inadequate performance of SimKGC~\cite{0046ZWL22} in inductive reasoning. 
This suggests a potential necessity for text-based and pretrained language model methods to have access to larger amounts of training data to converge and attain improved results.

\input{tables/ablation_results}

\subsection{Ablation Study}
In this part, we aim to evaluate the effectiveness of the proposed attention mechanism in \model. 
We conduct comparisons against several variations, including: (1) removing attention, which results in \model degenerating into vanilla path-based methods such as NBFNet~\cite{ZhuZXT21}; (2) excluding the query function; (3) excluding the value function; (4) substituting the kernel function with an approximation method based on random features~\cite{SinhaD16,LiuHCS22,WuZLWY22}; (5) substituting the kernel function with full-exponential kernel function.

The results of the ablation study on FB15k-237 are presented in \cref{tab:tab.3}.
We can initially observe a decline in model performance when attention is removed, underscoring the importance of the proposed attention mechanism. 
The performance decrease caused by omitting the query function is relatively limited, suggesting that the input of the attention layer retains some structural information after the update of previous layers.
Furthermore, omitting the value function leads to a significant decrease in model performance, as the explicit integration of structural information is crucial for reasoning tasks in knowledge graphs. 
In contrast, the pure transformer model faces difficulties in achieving this, highlighting the necessity of structural-aware modules. 
Lastly, a comparison with RF-based methods demonstrates the effectiveness of our kernel function.

To compare the performance between our proposed approximation kernel and the original full-exponential kernel, we conduct an additional ablation study on a small dataset UMLS~\cite{KokD07} in \cref{tab:app.tab.3} due to the computation overhead.
The minimal variation in the performance of both variants indicates the effectiveness of the approximate kernel. 
The full exponential kernel has a slight performance advantage, yet the quadratic computational complexity it brings is unacceptable in knowledge graph reasoning.

\begin{figure}[tb]
    \centering
    \includegraphics[scale=0.4]{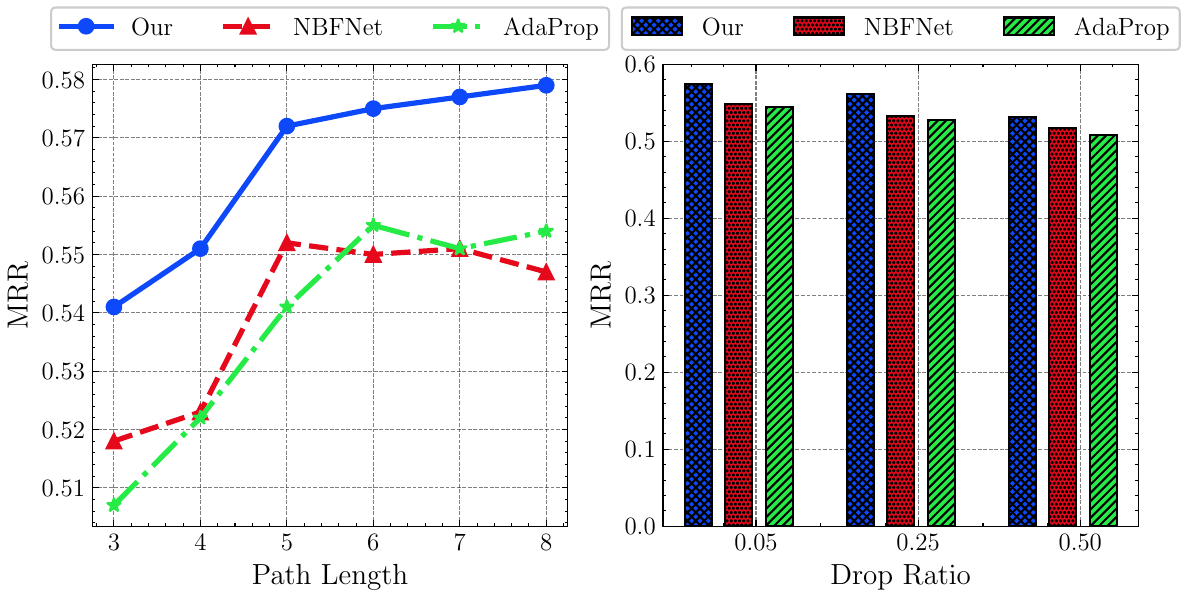}
    \caption{Experimental results on WN18RR. In the left chart, we evaluate the performance of all methods under different lengths of reasoning paths. In the right chart, we randomly drop some facts from the test data and report the performance of all methods.}
    \label{fig:fig.2}
    \vspace{-1.5em}
\end{figure}

\subsection{Further Analysis}
In this section, we present additional experiments to demonstrate the superiority of \model over vanilla path-based methods~\cite{ZhuZXT21,ZhangY22}. 
Specifically, we concentrate on two folds: (1) comparing methods in scenarios with gradually growing reasoning paths, and (2) comparing methods in scenarios with missing reasoning paths, where we randomly drop paths between pairs of entities in the test data. The experimental results on the WN18RR dataset are illustrated in \cref{fig:fig.2}.

\paragraph{Performance w.r.t. Longer Paths.}
The individual performance of \model remained unaffected even as the reasoning path extended, validating its capacity to mitigate information over-squashing resulting from prolonged entity interactions. 
This finding is consistent with the results in \citet{0002Y21}.

\paragraph{Performance w.r.t. Missing Paths.}
As the proportion of missing paths increases, \model exhibits stronger robustness compared to the baseline methods, which show a noticeable performance decline. 
This demonstrates that the presence of all-pair interactions in \model enables it to better handle missing paths.

%% file: tables/inductive_main_results.tex
\begin{table*}[htb]
    \centering

    \begin{adjustbox}{width=\textwidth}
    
    \begin{tabular}{lcccccccccccc}
    \toprule
    
    \multicolumn{1}{c}{\multirow{2}[4]{*}{\bf{Method}}} & \multicolumn{3}{c}{\bf{v1}} & \multicolumn{3}{c}{\bf{v2}} & \multicolumn{3}{c}{\bf{v3}} & \multicolumn{3}{c}{\bf{v4}}   \\
    
    \cmidrule(r){2-4} \cmidrule(r){5-7} \cmidrule(r){8-10} \cmidrule(r){11-13}  
    & \multicolumn{1}{c}{MRR} & \multicolumn{1}{c}{H@1} & \multicolumn{1}{c}{H@10} & \multicolumn{1}{c}{MRR} & \multicolumn{1}{c}{H@1} & \multicolumn{1}{c}{H@10} & \multicolumn{1}{c}{MRR} & \multicolumn{1}{c}{H@1} & \multicolumn{1}{c}{H@10} & \multicolumn{1}{c}{MRR} & \multicolumn{1}{c}{H@1} & \multicolumn{1}{c}{H@10}   \\

    \midrule

    \multicolumn{13}{c}{\bf{FB15k-237}} \\
    \midrule
    
    DRUM{\small~\cite{SadeghianADW19}}  & 0.333 & 24.7 & 47.4 & 0.395 & 28.4 & 59.5 & 0.402 & 30.8 & 57.1 & 0.410 & 30.9 & 59.3 \\
    NBFNet{\small~\cite{ZhuZXT21}} & 0.442 & 33.5 & 57.4 & \underline{0.514} & \underline{42.1} & \underline{68.5} & \underline{0.476} & 38.4 & \underline{63.7} & 0.453 & 36.0 & 62.7 \\
    RED-GNN{\small~\cite{ZhangY22}} & 0.369 & 30.2 & 48.3 & 0.469 & 38.1 & 62.9 & 0.445 & 35.1 & 50.3 & 0.442 & 34.0 & 62.1 \\
    A*Net{\small~\cite{zhu2023net}} & \underline{0.457} & \bf{38.1} & \underline{58.9} & 0.510 & 41.9 & 67.2 & \underline{0.476} & \underline{38.9} & 62.9 & \underline{0.466} & \underline{36.5} & \underline{64.5} \\
    AdaProp{\small~\cite{ZhangZY0023}} & 0.310 & 19.1 & 55.1 & 0.471 & 37.2 & 65.9 & 0.471 & 37.7 & \underline{63.7} & 0.454 & 35.3 & 63.8 \\
    \textsc{InGram}{\small~\cite{LeeCW23}} & 0.293 & 16.7 & 49.3 & 0.274 & 16.3 & 48.2 & 0.233 & 14.0 & 40.8 & 0.214 & 11.4 & 39.7 \\

    \hdashline
    \rowcolor{Tan!20}
    \bf{\model} & \bf{0.466} & \underline{37.8} & \bf{60.6} & \bf{0.532} & \bf{43.3} & \bf{70.3} & \bf{0.494} & \bf{40.0} & \bf{65.9} & \bf{0.480} & \bf{38.3} & \bf{65.3} \\

    \midrule
    \multicolumn{13}{c}{\bf{WN18RR}} \\
    \midrule

    DRUM{\small~\cite{SadeghianADW19}} & 0.666 & 61.3 & 77.7 & 0.646 & 59.5 & 74.7 & 0.380 & 33.0 & 47.7 & 0.627 & 58.6 & 70.2 \\
    NBFNet{\small~\cite{ZhuZXT21}} & \underline{0.741} & \underline{69.5} & \bf{82.6} & 0.704 & 65.1 & 79.8 & 0.452 & 39.2 & 56.8 & 0.641 & 60.8 & 69.4 \\  
    RED-GNN{\small~\cite{ZhangY22}} & 0.701 & 65.3 & 79.9 & 0.690 & 63.3 & 78.0 & 0.427 & 36.8 & 52.4 & 0.651 & 60.6 & 72.1 \\
    A*Net{\small~\cite{zhu2023net}} & 0.727 & 68.2 & 81.0 & 0.704 & \underline{64.9} & 80.3 & 0.441 & 38.6 & 54.4 & \underline{0.661} & \bf{61.6} & \underline{74.3} \\
    AdaProp{\small~\cite{ZhangZY0023}} & 0.733 & 66.8 & 80.6 & \bf{0.715} & 64.2 & \bf{82.6} & \bf{0.474} & \underline{39.6} & \bf{58.8} & \bf{0.662} & \underline{61.1} & \bf{75.5} \\
    \textsc{InGram}{\small~\cite{LeeCW23}} & 0.277 & 13.0 & 60.6 & 0.236 & 11.2 & 48.0 & 0.230 & 11.6 & 46.6 & 0.118 & 4.1 & 25.9 \\
    SimKGC{\small~\cite{0046ZWL22}} & 0.315 & 19.2 & 56.7 & 0.378 & 23.9 & 65.0 & 0.303 & 18.6 & 54.3 & 0.308 & 17.5 & 57.7 \\

    \hdashline
    \rowcolor{Tan!20}
    \bf{\model} & \bf{0.752} & \bf{71.5} & \underline{81.9} & \underline{0.709} & \bf{65.6} & \underline{81.7} & \underline{0.467} & \bf{40.6} & \underline{57.1} & 0.646 & 60.9 & 72.7 \\

    \midrule
    \multicolumn{13}{c}{\bf{NELL-995}} \\
    \midrule

    NBFNet{\small~\cite{ZhuZXT21}} & 0.584 & 50.0 & 79.5 & 0.410 & 27.1 & 63.5 & 0.425 & 26.2 & 60.6 & 0.287 & 25.3 & 59.1 \\  
    RED-GNN{\small~\cite{ZhangY22}} & 0.637 & 52.2 & 86.6 & 0.419 & 31.9 & 60.1 & \underline{0.436} & \underline{34.5} & 59.4 & 0.363 & \underline{25.9} & \bf{60.7} \\
    AdaProp{\small~\cite{ZhangZY0023}} & 0.644 & 52.2 & \underline{88.6} & \underline{0.452} & \underline{34.4} & \underline{65.2} & 0.435 & 33.7 & \underline{61.8} & \underline{0.366} & 24.7 & \bf{60.7} \\
    \textsc{InGram}{\small~\cite{LeeCW23}} & \underline{0.697} & \underline{57.5} & 86.5 & 0.358 & 25.3 & 59.6 & 0.308 & 19.9 & 50.9 & 0.221 & 12.4 & 44.0 \\

    \hdashline
    \rowcolor{Tan!20}
    \bf{\model} & \bf{0.827} & \bf{77.0} & \bf{93.0} & \bf{0.465} & \bf{35.7} & \bf{65.7} & \bf{0.478} & \bf{37.8} & \bf{65.7} & \bf{0.378} & \bf{26.7} & \underline{59.8} \\
    
    \bottomrule
    \end{tabular}
    \end{adjustbox}
    
  \caption{Inductive knowledge graph reasoning performance for 3 different datasets and 12 different versions. For each version, the best results are \textbf{boldfaced} and the second-best results are \underline{underlined}. Our proposed model, \model is marked by {\colorbox{Tan!20}{\textbf{tan}}}.}
  \label{tab:tab.2}
\end{table*}

%% file: tables/ablation_results.tex
\begin{table}[tb]
    \centering
    \small

     \begin{adjustbox}{width=.45\textwidth}
    
    \begin{tabular}{lccc}
    \toprule
    
    {\bf Method} & \multicolumn{1}{c}{{\bf MRR}} & \multicolumn{1}{c}{{\bf H@1}} & \multicolumn{1}{c}{{\bf H@10}}  \\

    \midrule
    \multicolumn{4}{c}{FB15k-237} \\
    \midrule

    \rowcolor{Tan!20}
    \bf{\model} & \bf{0.430} & \bf{34.3} & \bf{60.8} \\

    \;\;\textit{w/o} attention & 0.417 & 32.8 & 58.8 \\
    \;\;\textit{w/o} query function & 0.422 & 33.0 & 59.2 \\
    \;\;\textit{w/o} value function  & 0.367 & 30.7 & 48.7 \\
    \;\;RF-based kernel function  & 0.419 & 32.9 & 57.6 \\

    \midrule
    \multicolumn{4}{c}{UMLS} \\
    \midrule

    \rowcolor{Tan!20}
    \bf{\model} & 0.971 & \bf{95.8} & \bf{99.8} \\
    \;\;full-exponential kernel function & \bf{0.974} & 95.2 & 99.5 \\
    
    \bottomrule
    \end{tabular}

    \end{adjustbox}
    
  \caption{Ablation study of \model and its variants on transductive FB15k-237 dataset and UMLS dataset. We mark the vanilla \model by {\colorbox{Tan!20}{\textbf{tan}}}. The best performance is \textbf{boldfaced}.}
  \label{tab:tab.3}
  \vspace{-1.5em}
\end{table}

%% file: sections/6.conclusion.tex
This paper proposes a novel transformer-based method \model for knowledge graph reasoning.
\model consists of an expressive and scalable attention mechanism.
Specifically, we introduce message-passing neural network based query function and value function to obtain informative key and value representations.
Additionally, we present an efficient attention computation method to enhance the scalability of \model.
Experimental results show that \model outperforms salient baselines on both transductive and inductive benchmarks.

%% file: sections/appendix.tex
\section{More Preliminaries \& Backgrounds}
\paragraph{Transformer Architecture.} \label{sec:sec.3.3}
The Transformer architecture is originally introduced in ~\citet{VaswaniSPUJGKP17}.
A vanilla transformer block consists of two main modules: a self-attention module followed by a feed-forward neural network.
In the self-attention module, the input feature matrix denoted as $\boldsymbol{X} \in \mathbb{R}^{n\times d}$ are projected to query $\boldsymbol{Q}$, key $\boldsymbol{K}$ and value $\boldsymbol{V}$ using linear projection matrices $\boldsymbol{W}_q \in \mathbb{R}^{d^\prime \times n}$, $\boldsymbol{W}_k \in \mathbb{R}^{d^\prime \times n}$ and $\boldsymbol{W}_v \in \mathbb{R}^{d^\prime \times n}$.
This is done by $\boldsymbol{Q} = \boldsymbol{X}\boldsymbol{W}_{query}$, $\boldsymbol{K} = \boldsymbol{X}\boldsymbol{W}_{key}$ and $\boldsymbol{V} = \boldsymbol{X}\boldsymbol{W}_{value}$, where the bias is omited.
The self-attention computation is then given by:
\begin{equation} \label{eq.1}
    \text{Self-Attention}(\boldsymbol{X}) = \text{Softmax}\left(\frac{\boldsymbol{Q}\boldsymbol{K}^T}{\sqrt{d^\prime}}\right) \boldsymbol{V}.
\end{equation}
In practice, it is common to utilize \textit{multi-head} attention, where multiple instances of \cref{eq.1} are concatenated. This approach has demonstrated effectiveness in capturing multi-view interactions.
Subsequently, the output of the self-attention is combined with layer normalization~\cite{BaKH16} and a feedforward network (FFN) to form a transformer block.

\paragraph{Graph Transformers.}
Recently, transformer models~\cite{VaswaniSPUJGKP17} have gained popularity in graph learning due to their ability to capture complex relationships beyond those captured by regular graph neural networks (GNNs) differently.
Injecting structural bias into the original attention mechanism is a key problem in graph transformers. 
Early work by \citet{abs-2012-09699} used Laplacian eigenvectors as positional encodings, and since then, several extensions and other models have been proposed~\cite{abs-2202-08455,KimNMCLLH22,Ma0LRDCTL23}.
\citet{WuJWMGS21} propose a hybrid architecture that uses a stack of message-passing GNN layers followed by regular transformer layers.
The most relevant work to this paper is SAT~\cite{ChenOB22}, which reformulates the self-attention mechanism as a smooth kernel and incorporates structural information by extracting a subgraph representation rooted at each node before attention computation. Computation during the process follows the equations displayed below:
\begin{equation}
    \text{Self-Attention}(\boldsymbol{x}_u) = \sum_{v\in \mathcal{V}} \frac{\kappa(\boldsymbol{x}_u, \boldsymbol{x}_v)}{\sum_{w\in \mathcal{V}}\kappa(\boldsymbol{x}_u, \boldsymbol{x}_w)} f(\boldsymbol{x}_v), \forall u \in \mathcal{V},
\end{equation}
where $f(\boldsymbol{x})=\boldsymbol{x}\boldsymbol{W}_{value} + b_{value}$, and the non-symmetric kernel $\kappa$ is defined as:
\begin{equation}
    \kappa(\boldsymbol{x}_u, \boldsymbol{x}_v) = \exp \left( \frac{\left\langle \textrm{GNN}(\boldsymbol{x}_u) \boldsymbol{W}_{query} + b_{query}, \textrm{GNN}(\boldsymbol{x}_v)\boldsymbol{W}_{key} + b_{key} \right\rangle}{\sqrt{d}} \right).
\end{equation}
Moreover, other works focus on the development of scalable models, such as NodeFormer~\cite{WuZLWY22}, GraphGPS~\cite{RampasekGDLWB22}, DIFFormer~\cite{WuYZHWY23}, and SGFormer~\cite{wu2023simplifying}.

\section{Proof of \cref{theo:theo.1}} \label{app:app.theo1}

\begin{proof} 
Recall that we have the following equations:
\begin{equation}
    \begin{aligned}
        &\kappa(\boldsymbol{\widetilde{z}}_u,\boldsymbol{\widetilde{z}}_v) = 1 +  \left\langle \frac{\boldsymbol{\widetilde{z}}_u\boldsymbol{W}_1}{\Vert \boldsymbol{\widetilde{z}}_u\boldsymbol{W}_1 \Vert_{\mathcal{F}}}, \frac{\boldsymbol{\widetilde{z}}_v\boldsymbol{W}_2}{\Vert \boldsymbol{\widetilde{z}}_v\boldsymbol{W}_2 \Vert_{\mathcal{F}} } \right\rangle, \\
        &\kappa_{\text{exp}}(\boldsymbol{\widetilde{z}}_u,\boldsymbol{\widetilde{z}}_v) = \exp \left( \left\langle \frac{\boldsymbol{\widetilde{z}}_u\boldsymbol{W}_1}{\Vert \boldsymbol{\widetilde{z}}_u\boldsymbol{W}_1 \Vert_{\mathcal{F}}}, \frac{\boldsymbol{\widetilde{z}}_v\boldsymbol{W}_2}{\Vert \boldsymbol{\widetilde{z}}_v\boldsymbol{W}_2 \Vert_{\mathcal{F}} } \right\rangle \right),
    \end{aligned}
\end{equation}
where we omit the bias. For convenience, we define $\boldsymbol{u} = \frac{\boldsymbol{\widetilde{z}}_u\boldsymbol{W}_1}{\Vert \boldsymbol{\widetilde{z}}_u\boldsymbol{W}_1 \Vert_{\mathcal{F}}}$ and $\boldsymbol{v} = \frac{\boldsymbol{\widetilde{z}}_v\boldsymbol{W}_2}{\Vert \boldsymbol{\widetilde{z}}_v\boldsymbol{W}_2 \Vert_{\mathcal{F}}}$, yielding:
\begin{equation}
    \begin{aligned}
        &\kappa(\boldsymbol{\widetilde{z}}_u,\boldsymbol{\widetilde{z}}_v) = 1 +  \left\langle \boldsymbol{u},\boldsymbol{v} \right\rangle, \\
        &\kappa_{\text{exp}}(\boldsymbol{\widetilde{z}}_u,\boldsymbol{\widetilde{z}}_v) = \exp \left( \left\langle \boldsymbol{u},\boldsymbol{v} \right\rangle \right),
    \end{aligned}
\end{equation}
where $\langle \boldsymbol{u},\boldsymbol{v} \rangle \in [-1, 1]$. According to the Taylor formula, the expression for $\kappa_{\text{exp}}(\boldsymbol{\widetilde{z}}_u,\boldsymbol{\widetilde{z}}_v)$ is as follows:
\begin{equation} \label{eq:eq.lagr}
    \exp \left( \left\langle \boldsymbol{u},\boldsymbol{v} \right\rangle \right) = \underbrace{1 + \left\langle \boldsymbol{u},\boldsymbol{v} \right\rangle}_{\text{first-order Taylor expansion}} + \underbrace{\frac{\exp (\xi)}{2} (\left\langle \boldsymbol{u},\boldsymbol{v} \right\rangle)^2}_{\text{Lagrange remainder term}},
\end{equation}
where $\xi \in (0, \left\langle \boldsymbol{u},\boldsymbol{v} \right\rangle)$. 
In other words, we have:
\begin{equation}
    \left\lvert \kappa(\boldsymbol{\widetilde{z}}_u,\boldsymbol{\widetilde{z}}_v) -  \kappa_{\text{exp}}(\boldsymbol{\widetilde{z}}_u,\boldsymbol{\widetilde{z}}_v) 
  \right\rvert = \frac{\exp (\xi)}{2} (\left\langle \boldsymbol{u},\boldsymbol{v} \right\rangle)^2.
\end{equation}
Further, by taking $\xi = \gamma \cdot \left\langle \boldsymbol{u},\boldsymbol{v} \right\rangle$, where $\gamma \in (0, 1)$, we can rewrite the Lagrange remainder term in \cref{eq:eq.lagr} as:
\begin{equation}
    \frac{\exp (\xi)}{2} (\left\langle \boldsymbol{u},\boldsymbol{v} \right\rangle)^2 = \frac{\exp (\gamma \cdot \left\langle \boldsymbol{u},\boldsymbol{v} \right\rangle)}{2} (\left\langle \boldsymbol{u},\boldsymbol{v} \right\rangle)^2.
\end{equation}
Then we have:
\begin{equation}
    \begin{aligned}
        \left\lvert \kappa(\boldsymbol{\widetilde{z}}_u,\boldsymbol{\widetilde{z}}_v) -  \kappa_{\text{exp}}(\boldsymbol{\widetilde{z}}_u,\boldsymbol{\widetilde{z}}_v) 
        \right\rvert = \frac{\exp (\gamma \cdot \left\langle \boldsymbol{u},\boldsymbol{v} \right\rangle)}{2} (\left\langle \boldsymbol{u},\boldsymbol{v} \right\rangle)^2 & < \;\max(e^{-\gamma}/2,e^{\gamma}/2) \\
        & < \;e^{\gamma}/2.
    \end{aligned}
\end{equation}
\end{proof}

\section{Details of Relational Asymmetric Local 2-WL Test} \label{app:app.2wl}
Firstly, we present the general form of the Weisfeiler-Lehman (WL) Test~\cite{weisfeiler1968reduction}. 
Let $\mathcal{G} = \{\mathcal{V}, \mathcal{E}, c\}$ denotes a graph where $\mathcal{V}$ represents the set of nodes, $\mathcal{E}$ represents the set of edges, and $c$ represents the node coloring. The node coloring, also known as the feature mapping, is denoted as $c: \mathcal{V} \rightarrow \mathbb{R}^d$. 
The purpose of the Weisfeiler-Lehman (\texttt{WL}) test is to detect graph isomorphism.
\begin{restatable}[Graph Isomorphism]{definition}{defgi}
    An \textit{isomorphism} between a graph $\mathcal{G} = \{\mathcal{V},\mathcal{E}, c\}$ and $\mathcal{G}^\prime = \{\mathcal{V}^\prime,\mathcal{E}^\prime, c^\prime\}$ is a \textit{bijection} $f:\mathcal{V} \rightarrow \mathcal{V}^\prime$ that satisfies $c(v) = c^\prime(f(v))$ for all $v\in\mathcal{V}$, and $(u,v) \in \mathcal{E}$ if and only if $(f(u), f(v)) \in \mathcal{E}^\prime$ for all $u,v\in\mathcal{V}$.
\end{restatable}
\begin{restatable}[\texttt{WL} test]{definition}{defwlt}
    Consider a graph $\mathcal{G} = \{\mathcal{V},\mathcal{E}, c\}$.  The \texttt{WL} test is defined as follows:
    \begin{equation}
        \begin{aligned}
            &\texttt{WL}^{(0)}(u) = c(u), \\
            &\texttt{WL}^{(t+1)}(u) = \tau\left( \texttt{WL}^{(t)}(u), \left\ldblbrace \texttt{WL}(v)\;|\; v\in\mathcal{N}(u) \right\rdblbrace \right),
        \end{aligned}
    \end{equation}
    where $\ldblbrace\cdot\rdblbrace$ denotes a multiset, $\mathcal{N}(u)$ means the neighborhood of node $u$, and $\tau$ maps the pair above injectively to a unique color that has not been used in previous iterations.
    Repeat the above steps $T$ times until convergence.
\end{restatable}
In general, the \texttt{WL} test can be used as a necessary but not sufficient condition for detecting graph isomorphism. 
Additionally, there exists the \texttt{WL}$_k$ ($k>1$) test~\cite{XuHLJ19}, which operates on $k$-tuples of nodes $\boldsymbol{u}\in \mathcal{V}^k$ and provides enhanced expressive power.

Now let us delve into the realm of knowledge graphs. 
Consider a knowledge graph denoted by $\mathcal{G} = \{\mathcal{V}, \mathcal{E}, \mathcal{R}, c\}$, where $\mathcal{V}$ represents the set of entities, $\mathcal{R}$ represents the set of relations, $\mathcal{E}$ represents the set of facts, and $c$ represents the entity coloring scheme. 
Similarly, we provide the definition for knowledge graph isomorphism.
\begin{restatable}[Knowledge Graph Isomorphism]{definition}{defkgi}
    An \textit{isomorphism} between knowledge graphs $\mathcal{G} = \{\mathcal{V},\mathcal{E}, \mathcal{R}, c\}$ and $\mathcal{G}^\prime = \{\mathcal{V}^\prime,\mathcal{E}^\prime, \mathcal{R}^\prime, c^\prime\}$ is a \textit{bijection} $f:\mathcal{V} \rightarrow \mathcal{V}^\prime$ that satisfies $c(v) = c^\prime(f(v))$ for all $v\in\mathcal{V}$, and $r(u,v) \in \mathcal{E}$ if and only if $r(f(u), f(v)) \in \mathcal{E}^\prime$ for all $r\in\mathcal{R}$ and $u,v\in\mathcal{V}$.
\end{restatable}
The \texttt{R-WL} test is a relational variant of the \texttt{WL} test, proposed by \citet{Barcelo00O22}. It is defined as follows.
\begin{restatable}[\texttt{R-WL} test]{definition}{defrwlt}
    Let $\mathcal{G} = \{\mathcal{V},\mathcal{E}, \mathcal{R}, c\}$ be a knowledge graph. The \texttt{R-WL} test is defined as:
    \begin{equation}
        \begin{aligned}
            &\texttt{R-WL}^{(0)}(u) = c(u), \\
            &\texttt{R-WL}^{(t+1)}(u) = \tau\left( \texttt{R-WL}^{(t)}(u), \left\ldblbrace (\texttt{R-WL}(v),r)\;|\; v\in\mathcal{N}_r(u),r\in\mathcal{R} \right\rdblbrace \right),
        \end{aligned}
    \end{equation}
    where $\ldblbrace\cdot\rdblbrace$ denotes a multiset, $\mathcal{N}_r(u)$ refers to the neighborhood of node $u$ corresponding to relation $r$, and $\tau$ injectively maps the aforementioned pair to a unique color not yet used in prior iterations.
\end{restatable}
However, the aforementioned tests are not suitable for quantifying the expressive power of methods employed in link prediction, as they require a measurement of a \textit{binary variant}.
\begin{restatable}[\textit{binary variant} on knowledge graphs]{definition}{bvokg}
    A \textit{binary variant} on knowledge graphs is represented as a function $\xi$ that associates each knowledge graph $\mathcal{G}=\{\mathcal{V}, \mathcal{E},\mathcal{R},c\}$ with a function $\xi(\mathcal{G})$ defined on $\mathcal{V}^2$. This function satisfies the condition that for all knowledge graphs $\mathcal{G}$ and $\mathcal{G}^\prime$, all isomorphisms $f$ from $\mathcal{G}$ to $\mathcal{G}^\prime$, and all $2$-tuples of entities $\boldsymbol{u} \in \mathcal{V}^2$, it holds that $\xi(\mathcal{G})(\boldsymbol{u}) = \xi(\mathcal{G}^\prime)(f(\boldsymbol{u}))$.
\end{restatable}
To analyze the expressive power of knowledge graph reasoning methods, \citet{huang2023theory} introduces the \textit{relational asymmetric local 2-WL} test denoted by \texttt{RA-WL}$_2$.
Specifically, \texttt{RA-WL}$_2$ is defined on a knowledge graph $\mathcal{G} = \{ \mathcal{V}, \mathcal{E}, \mathcal{R}, c, \eta \}$ with a \textit{pairwise coloring} $\eta:\mathcal{V}\times\mathcal{V}\rightarrow\mathcal{D}$ that satisfies \textit{target node distinguishability}, meaning $\eta(u,u) \ne \eta(u,v)$ for all $u\ne v \in \mathcal{V}$.
Then \texttt{RA-WL}$_2$ is defined as:
\begin{restatable}[\texttt{RA-WL}$_2$ test]{definition}{rawlt}
    Let $\mathcal{G} = \{\mathcal{V},\mathcal{E}, \mathcal{R}, c\}$ be a knowledge graph. The \texttt{RA-WL}$_2$ test is defined as:
    \begin{equation}
        \begin{aligned}
            &\texttt{RA-WL}_2^{(0)}(u,v) = \eta(u,v), \\
            &\texttt{RA-WL}_2^{(t+1)}(u,v) = \tau\left( \texttt{RA-WL}_2^{(t)}(u,v), \left\ldblbrace (\texttt{RA-WL}_2(u,w),r)\;|\; w\in\mathcal{N}_r(u),r\in\mathcal{R} \right\rdblbrace \right),
        \end{aligned}
    \end{equation}
    where $\ldblbrace\cdot\rdblbrace$ denotes a multiset, $\mathcal{N}_r(u)$ means the neighborhood of node $u$ corresponding to relation $r$, and $\tau$ maps the pair above injectively to a unique color that has not been used in previous iterations.
\end{restatable}
\texttt{RA-WL}$_2$ test allows us to characterize the power of methods in terms of their ability to distinguish pairs of entities on knowledge graphs.

\section{Proof of \cref{theo:theo.2}} \label{app:app.theo2}
\begin{proof}
    To prove \cref{theo:theo.2}, we begin by establishing the equivalence between the query function $f_q(\cdot)$ and \texttt{RA-WL}$_2$ on original knowledge graph $\mathcal{E}$.
    The proof is mainly based on the results in \citet{huang2023theory}.
    As stated in \citet{huang2023theory}, a conditional message-passing neural network for inferring the query $(e_h, r_q, ?)$ is defined as follows:
    \begin{equation} \label{eq:eq.c.1}
        \begin{aligned}
            &\boldsymbol{h}_{u|e_h,r_q}^{(0)} = \textsc{Init}(u|e_h,r_q), \\
            &\boldsymbol{h}_{u|e_h,r_q}^{(t+1)} = \textsc{Upd}\left( \boldsymbol{h}_{u|e_h,r_q}^{(t)}, \textsc{Agg} \left(\left\{\mskip-5mu\left\{ \textsc{Msg} \left(\boldsymbol{h}_{v|e_h,r_q}^{(t)}, \boldsymbol{w}_{r_q} \right)|v\in\mathcal{N}_r(u),r\in\mathcal{R} \right\}\mskip-5mu\right\} \right), \textsc{Read} \left(\left\{\mskip-5mu\left\{ \boldsymbol{h}_{w|e_h,r_q}^{(t)}|w\in\mathcal{V} \right\}\mskip-5mu\right\} \right) \right),
        \end{aligned}
    \end{equation}
    where $\boldsymbol{h}_{u|e_h,r_q}^{(t)}$ represents a pairwise representation corresponding to $\boldsymbol{h}_{r_q}^{(t)}(e_h,u):\mathcal{V}\times\mathcal{V}\rightarrow \mathbb{R}^d$. Here, $\textsc{Init}$, $\textsc{Upd}$, $\textsc{Agg}$, $\textsc{Read}$, and $\textsc{Msg}_r$ are differentiable functions responsible for \textit{initialization}, \textit{update}, \textit{aggregation}, \textit{global readout}, and \textit{relation-specific message} computations, respectively. 
    Based on Theorem 5.1 in \citet{huang2023theory}, it is suggested that a message-passing neural network with an architecture equivalent to \cref{eq:eq.c.1}, featuring a \textit{target node distinguishable} \textsc{Init}, exhibits the same expressivity as \texttt{RA-WL}$_2$.

    Now we establish the equivalence of our proposed query function and \cref{eq:eq.c.1}. 
    \begin{lemma} \label{lem:lem.c.1}
        The query function $f_q(\cdot)$ shares the same architecture as the conditional message-passing neural network defined in \cref{eq:eq.c.1}.
    \end{lemma}
    \begin{proof}
        From \cref{eq:eq.6}, we derive the following:
        \begin{equation}
            \begin{aligned}
                \underbrace{\boldsymbol{z}_u^{(0)}}_{\boldsymbol{h}_{u|e_h,r_q}^{(0)}} &\leftarrow \underbrace{[\boldsymbol{x}_u, \boldsymbol{\epsilon}]}_{\textsc{Init}}, \\
                \underbrace{\boldsymbol{z}_u^{(l)}}_{\boldsymbol{h}_{u|e_h,r_q}^{(t+1)}} &\leftarrow \underbrace{\Phi \left(\boldsymbol{\alpha}^{(l-1)} \cdot \boldsymbol{z}_u^{(l-1)} + \underbrace{\sum_{r(v,u) \in \mathcal{E}} \underbrace{\boldsymbol{m}_{u|v,r}^{(l-1)}}_{\textsc{Msg}}}_{\textsc{Agg}}  \right)}_{\textsc{Upd}},
            \end{aligned}
        \end{equation}
        where we have omitted the $\textsc{Read}$ function in our query function.
    \end{proof}
    We will now demonstrate that the $\textsc{Init}$ function in our query function ensures \textit{target node distinguishability}, which is defined as $\textsc{Init}(u|u,r_q) \ne \textsc{Init}(v|u,r_q)$ if $u \ne v$ for each $v \in \mathcal{V}\setminus \{u\}$.
    \begin{lemma} \label{lem:lem.c.2}
        For each $v \in \mathcal{V}\setminus \{u\}$, there exists $\lim_{d \rightarrow \infty} P(\boldsymbol{z}_{u}^{(0)} = \boldsymbol{z}_{v}^{(0)}) = 0$, where $\boldsymbol{z}_u^{(0)},\boldsymbol{z}_v^{(0)}\in\mathbb{R}^d$.
    \end{lemma}
    \begin{proof}
        For simplicity, we consider only the random component of $\boldsymbol{z}^{(0)} \in \mathbb{R}^d$, assuming $\boldsymbol{z}^{(0)}[i] \sim N(0,1)$. Let $\boldsymbol{z}_u^{(0)} = [z_{u,0},z_{u,1},\ldots,z_{u,d}]$ and $\boldsymbol{z}_v^{(0)} = [z_{v,0},z_{v,1},\ldots,z_{v,d}]$. We can calculate:
        \begin{equation}
            \begin{aligned}
                P(\boldsymbol{z}_{u}^{(0)} = \boldsymbol{z}_{v}^{(0)}) &= \prod_{i = 0}^d P(z_{u,i} = z_{v,i}) \\
                &= \prod_{i = 0}^d \frac{1}{2\pi} \exp\left(-z_{u,i}^2\right) \\
                &\le \prod_{i = 0}^d \frac{1}{2\pi} \\
                &= \left(\frac{1}{2\pi}\right)^d.
            \end{aligned}
        \end{equation}
        Furthermore, using $\frac{1}{2\pi} < 1$, we find $\lim_{d \rightarrow \infty} P\left(\boldsymbol{z}_{u}^{(0)} = \boldsymbol{z}_{v}^{(0)}\right) = 0$.
    \end{proof}
    From Lemma \ref{lem:lem.c.2}, it follows that our $\textsc{Init}$ function ensures \textit{target node distinguishability} when employing a sufficiently large dimension $d$ for random features.
    Based on \cref{lem:lem.c.1} and \cref{lem:lem.c.2}, we can derive the following lemma:
    \begin{lemma} \label{lem:lem.c.3}
        The proposed query function $f_q(\cdot)$ can be as expressive as \texttt{RA-WL}$_2$ on graph $\mathcal{E}$.
    \end{lemma}
    Since the initialization in \cref{eq:eq.7} also satisfies the \textit{target node distinguishability}. 
    By employing a similar approach, we can also derive the following lemma.
    \begin{lemma} \label{lem:lem.c.3.1}
        The proposed value function $f_v(\cdot)$ can be as expressive as \texttt{RA-WL}$_2$ on graph $\mathcal{E}$.
    \end{lemma}
    To further showcase the expressivity of our proposed attention function on the estimated graph $\mathcal{\tilde{E}}$, we aim to establish the injectiveness of our attention function. 
    The update process of \model's attention can be reformulated as:
    \begin{equation}
        \boldsymbol{z}_u = \phi \left( \xi \cdot \boldsymbol{\overline{z}}_u, f\left( \left\{ \boldsymbol{\hat{z}}_v : v \in \mathcal{V} \right\} \right)  \right),
    \end{equation}
    where $\boldsymbol{\overline{z}}_u$ means the output of the value function, $\left\{ \boldsymbol{\hat{z}}_v : v \in \mathcal{V} \right\}$ indicates the weighted values $\alpha_{uv} \boldsymbol{\overline{z}}_v$ where $\alpha_{uv}$ is the attention score and $\xi$ is a constant factor which is equal to $\lvert \mathcal{V} \rvert$ here.
    
    We denote the space of $\boldsymbol{\hat{z}}$ as $\mathcal{Z}$. 
    Recall that the query function $f_q(\cdot)$ shares the same expressive capacity as \texttt{RA-WL}$_2$, enabling it to distinguish the diverse neighbor structures associated with each $u\in \mathcal{V}$. 
    Consequently, for distinct $u$ values, we can obtain diverse multisets $\left\{ \boldsymbol{\hat{z}}_v : v \in \mathcal{V} \right\} \subset \mathcal{Z}$ which forms the estimated graph $\mathcal{\tilde{E}}$.

    Further, we have the following theorem:
    \begin{lemma} \label{lem:lem.c.4}
        Assuming $\mathcal{Z}$ is countable, the update process of \model's attention is injective to each pair $(\boldsymbol{\overline{z}}_v, \left\{ \boldsymbol{\hat{z}}_v : v \in \mathcal{V} \right\})$.
    \end{lemma}
    \begin{proof}
        To prove \cref{lem:lem.c.4}, we express the update process as follows:
        \begin{equation} \label{lem:lem.c.4.eq}
            \boldsymbol{z}_u = g \left( (1 + \zeta) \cdot f(\boldsymbol{\overline{z}}_u) + \sum_{v\in\mathcal{V}} f(\boldsymbol{\hat{z}}_v)  \right).
        \end{equation}
        It has been shown in Lemma 5 and Corollary 6 in \citet{XuHLJ19} that there exists a function $f: \mathcal{Z} \rightarrow \mathbb{R}^d$ for which \cref{lem:lem.c.4.eq} is unique for $(\boldsymbol{\overline{z}}_v, \left\{ \boldsymbol{\hat{z}}_v : v \in \mathcal{V} \right\})$.
        The FFN following the attention can be employed to model $g \circ f$ based on the universal approximation theorem~\cite{HornikSW89,Hornik91}.
        Consequently, \cref{lem:lem.c.4} can be concluded.
    \end{proof}
    \cref{lem:lem.c.4} provides evidence for the injectiveness of our proposed attention function on the estimated graph $\mathcal{\tilde{E}}$.
    From the definition of \texttt{RA-WL}$_2$, we can conclude that the attention layer of \model is capable of attaining the same level of expressiveness as \texttt{RA-WL}$_2$, thereby validating \cref{theo:theo.2}.
\end{proof}

\section{Experimental Details}
\subsection{Dataset Statistics}
We conduct experiments on four transductive knowledge graph reasoning datasets, and the statistics of these datasets are summarized in \cref{tab:app.tab.1}. 
Additionally, we perform experiments on three inductive knowledge graph reasoning datasets, each of which contains four different splits. The statistics of the inductive datasets are summarized in \cref{tab:app.tab.2}.

\input{tables/transductive_dataset_statistics}
\input{tables/inductive_dataset_statistics}

\subsection{Evaluation}
For each dataset, we evaluate the performance of ranking each answer entity against all negative entities given a query $(h,r,?)$. The evaluation metrics we use are the mean reciprocal rank (MRR)~\cite{BordesUGWY13} and Hits at $n$ (H@$n$)~\cite{BordesUGWY13}, which are computed as follows:
\begin{equation}
    \begin{aligned}
        &\text{MRR} = \frac{1}{|\mathcal{A}|} \sum_{i=1}^{|A|} \frac{1}{rank_i}, \\
        &\text{Hits at } n = \frac{1}{|\mathcal{A}|} \mathbb{I} [rank_i \le k],
    \end{aligned}
\end{equation}
where $\mathcal{A}$ indicates the answer entity set, $\mathbb{I} [\cdot]$ indicates the indicative function, and $rank_i$ indicates the rank of each answer entity.

\subsection{Implement Details}
Our model is comprised of three modules: the input layer, the attention layer, and the output layer. 
We present the details for each of them as follows.
\begin{itemize}
    \item The input layer consists of the initialization of entity representations and relation representations. We initialize entity representations as all-zero vectors $\boldsymbol{X} = [0]^{\lvert \mathcal{V} \rvert \times d}$ and relation representations as randomly initialized learnable vectors $\boldsymbol{R} \in \mathbb{R}^{\lvert \mathcal{R} \rvert \times d}$.
    \item The attention layer is composed of the attention function which is followed by a skip-connection, an FFN, and two normalization layers before and after the FFN.
    The attention function consists of three sub-modules as described in \cref{sec:sec.4.1}. Specifically, each attention function feeds entity representations $\boldsymbol{X}^{(l-1)}$ and relation representation $\boldsymbol{R}$ and outputs the updated entity representation $\boldsymbol{X}^{(l)}$.
    The MLPs in the attention function are implemented as three-layer MLPs with ReLU~\cite{GlorotBB11} activation. 
    Additionally, We implement FFN as a two-layer MLP with ReLU~\cite{GlorotBB11} activation.
    \item The output layer is a feed-forward layer for prediction, which maps the entity representations into the predicted scores.
    The scores will be used for ranking answer entities.
\end{itemize}
For each experiment, we employ a fixed random seed and run it multiple times to report the average performance.

\subsection{Hyperparameters Setup}
For each dataset, we performed hyperparameter tuning on the validation set. 
We considered different values for the learning rate ($lr$) from the set $\{1e-4, 5e-4, 1e-3, 5e-3\}$, weight decay ($wd$) from the set $\{0, 1e-6, 1e-5, 1e-4\}$, hidden dimension ($d$) from the set $\{16, 32, 64\}$, number of negative samples ($|[t^\prime]|$) from the set $\{2^6, 2^8, 2^{10}, 2^{12}, 2^{14}, 2^{16}\}$, number of layers for the query function ($\widetilde{L}$) from the set $\{1, 2, 3\}$, number of layers for the value function ($\widehat{L}$) from the set $\{1, 2, 3\}$, and number of layers for \model ($L$) from the set $\{1, 2, 3\}$. 

\subsection{Hardcore Configurations}
We conduct all experiments with:
\begin{itemize}
    \item Operating System: Ubuntu 22.04.3 LTS.
    \item CPU: Intel~(R) Xeon~(R) Platinum 8358 CPU @ 2.60GHz with 1TB DDR4 of Memory and Intel Xeon Gold 6148 CPU @ 2.40GHz with 384GB DDR4 of Memory.
    \item GPU: NVIDIA Tesla A100 SMX4 with 40GB of Memory and NVIDIA Tesla V100 SXM2 with 32GB of Memory.
    \item Software: CUDA 12.1, Python 3.9.14, PyTorch~\cite{PaszkeGMLBCKLGA19} 2.1.0.
\end{itemize}

\section{More Experimental Results}
\subsection{\model v.s. \textsc{Ultra}} \label{app:app.aer.ultra}
We present the experimental results of \model and \textsc{Ultra} in \cref{tab:app.tab.3}.
It can be observed that \model achieves comparable performance to \textsc{Ultra}. Notably, \textsc{Ultra} demonstrates state-of-the-art (SOTA) performance in certain datasets due to its utilization of extensive pretraining data. 
As mentioned previously, our method can be easily adapted to address full-inductive tasks with minor modifications, making it compatible with a pre-training and fine-tuning paradigm. 
We intend to explore this avenue in future work.

\input{tables/vs_ultra}

\begin{figure}[htb]
    \centering

    \subfigure[$(\texttt{Decca\_Records},\; \texttt{artist},\; ?)$]{
        \includegraphics[scale=0.69]{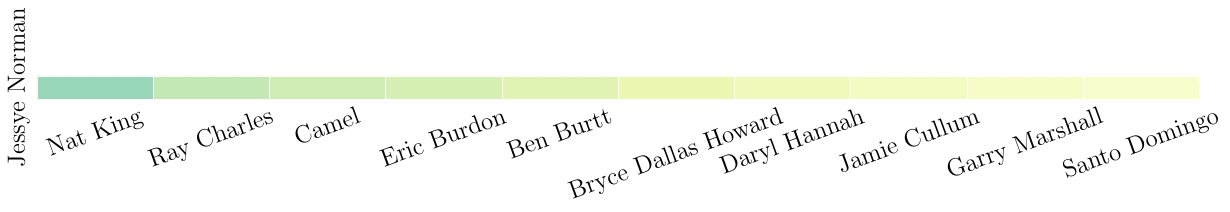}
    }
    \vspace{-0.3cm}
    \subfigure[$(\texttt{Egg},\; \texttt{nutrition\_fact},\; ?)$]{
        \includegraphics[scale=0.7]{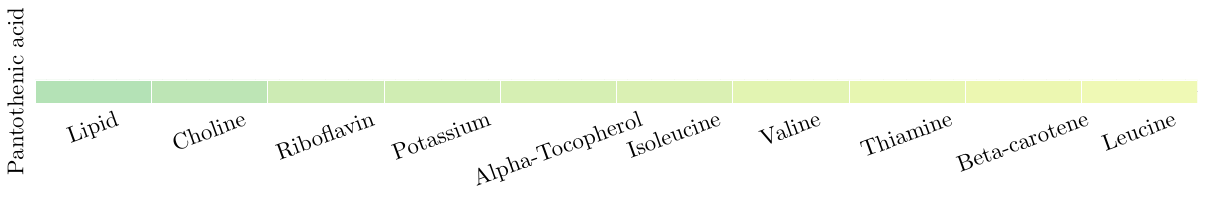}
    }
    \vspace{-0.3cm}
    \subfigure[$(\texttt{Jewish\_people},\; \texttt{ethnicity},\; ?)$]{
        \includegraphics[scale=0.7]{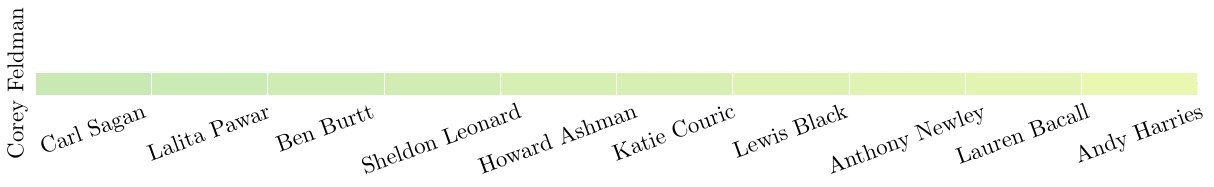}
    }
    \vspace{-0.3cm}
    \subfigure[$(\texttt{Piano},\; \texttt{music\_group},\; ?)$]{
        \includegraphics[scale=0.7]{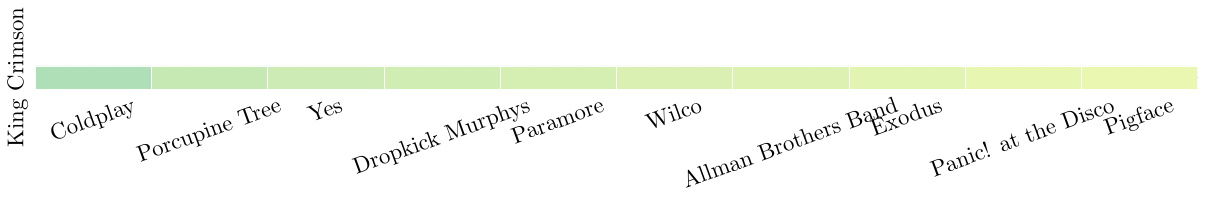}
    }
    \vspace{-0.3cm}
    \subfigure[$(\texttt{Bachelor's\_degree},\; \texttt{educational\_degree},\; ?)$]{
        \includegraphics[scale=0.7]{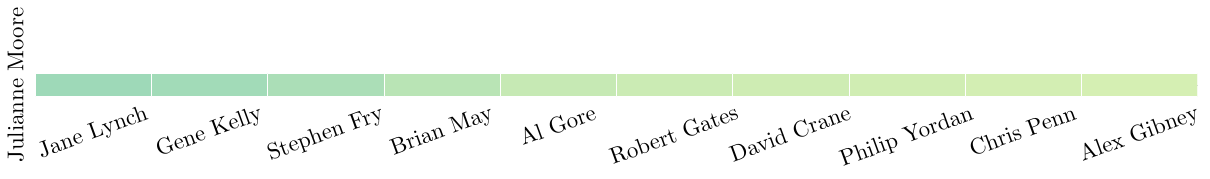}
    }
    
    \caption{Visualization of \model attention on FB15k-237 test set. We select the attention matrix corresponding to the answer entity for each test fact and visualize the top-$10$ entities, excluding the answer entity itself. }
    \label{app.fig.vis}
\end{figure}

\subsection{Performance of Homogeneous Graph Link Prediction}
Although \model is primarily designed for knowledge graph reasoning, it can also be applied to general link prediction tasks.
In this section, we present experimental results for the homogeneous graph link task using datasets including Cora, Citeseer, and PubMed~\cite{SenNBGGE08}.
We follow the same settings~(baselines and other experimental setup) as in \citet{ZhuZXT21}.
\cref{tab:app.tab.4} summarizes the results of \model and baselines on three datasets.
We can observe that \model surpasses the baselines on Cora and PubMed and achieves comparable performance on Citeseer, which demonstrates the effectiveness of \model on general graph link prediction tasks.

\input{tables/hom_link_prediction}

\subsection{Impact of the Number of Layers}
The number of layers, denoted as $L$, $\widetilde{L}$, and $\widehat{L}$, is a crucial hyperparameter in \model. 
To illustrate the influence of the number of layers, we present the experimental results on the transductive WN18RR dataset in \cref{tab:app.tab.5}.
We can observe that both the number of attention layers and the number of value function layers play a crucial role in determining the reasoning performance of \model. 
One notable finding is that increasing the number of attention layers can yield significant performance improvements, even in scenarios where the value function layers are limited. 
This finding further reinforces the effectiveness of the attention layers proposed in this paper.
Moreover, increasing the number of query function layers enhances the inference performance, particularly for shallow models~(\textit{e.g.}, limited number of attention layers and value function layers).
This observation suggests that the query function is adept at effectively capturing structural information.

\subsection{Case Study on \model's Interpretation}
To gain a deeper understanding of the attention function of \model, we present visualizations of the attention matrix in \cref{app.fig.vis}. 
These examples demonstrate that \model is capable of effectively capturing the inherent patterns among various entities, which contributes to the enhancement of its performance. 
For instance, when examining the test fact $(\texttt{Egg},\; \texttt{nutrition\_fact},\; ?)$, we observe that \model successfully distinguishes the nutrients of eggs, such as Lipid, Choline, and Riboflvain, aligning with human cognition.

\input{tables/impact_of_layers}

%% file: tables/transductive_dataset_statistics.tex
\begin{table}[htb]
    \centering
    \small
    \begin{tabular}{lccccc}
    \toprule
    
    \multicolumn{1}{c}{\multirow{2}[4]{*}{\bf{Dataset}}} & \multicolumn{1}{c}{\multirow{2}[4]{*}{\bf{\#Relation}}} & \multicolumn{1}{c}{\multirow{2}[4]{*}{\bf{\#Entity}}} & \multicolumn{3}{c}{\bf{\#Triplet}}    \\
    
    \cmidrule{4-6}  
    & & & \multicolumn{1}{c}{\#Train} & \multicolumn{1}{c}{\#Valid} & \multicolumn{1}{c}{\#Test}  \\

    \midrule

    FB15k-237 & 237 & 14,541 & 272,115 & 17,535 & 20,466 \\
    WN18RR & 11 & 40,943 & 86,835 & 3,034 & 3,134 \\
    NELL-995 & 200 & 74,536 & 149,678 & 543 & 2,818 \\
    YAGO3-10 & 37 & 123,182 & 1,079,040 & 5,000 & 5,000 \\
    
    \bottomrule
    \end{tabular}
    
  \caption{Dataset Statistics for transductive knowledge graph reasoning datasets.}
  \label{tab:app.tab.1}
\end{table}

%% file: tables/inductive_dataset_statistics.tex
\begin{table}[hbt]
    \centering
    \begin{adjustbox}{max width=.9\textwidth}
        \begin{tabular}{llcccccccccc}
            \toprule
            \multirow{2}{*}{\bf{Dataset}} & & \multirow{2}{*}{\bf{\#Relation}} & \multicolumn{3}{c}{\bf{Train}} & \multicolumn{3}{c}{\bf{Validation}} & \multicolumn{3}{c}{\bf{Test}} \\
            \cmidrule{4-12}
            & & & \#Entity & \#Query & \#Fact & \#Entity & \#Query & \#Fact & \#Entity & \#Query & \#Fact \\
            \midrule
            \multirow{4}{*}{FB15k-237}
            & v1 & 180 & 1,594 & 4,245 & 4,245 & 1,594 & 489 & 4,245 & 1,093 & 205 & 1,993 \\
            & v2 & 200 & 2,608 & 9,739 & 9,739 & 2,608 & 1,166 & 9,739 & 1,660 & 478 & 4,145 \\
            & v3 & 215 & 3,668 & 17,986 & 17,986 & 3,668 & 2,194 & 17,986 & 2,501 & 865 & 7,406 \\
            & v4 & 219 & 4,707 & 27,203 & 27,203 & 4,707 & 3,352 & 27,203 & 3,051 & 1,424 & 11,714 \\
            \midrule
            \multirow{4}{*}{WN18RR}
            & v1 & 9 & 2,746 & 5,410 & 5,410 & 2,746 & 630 & 5,410 & 922 & 188 & 1,618 \\
            & v2 & 10 & 6,954 & 15,262 & 15,262 & 6,954 & 1,838 & 15,262 & 2,757 & 441 & 4,011 \\
            & v3 & 11 & 12,078 & 25,901 & 25,901 & 12,078 & 3,097 & 25,901 & 5,084 & 605 & 6,327 \\
            & v4 & 9 & 3,861 & 7,940 & 7,940 & 3,861 & 934 & 7,940 & 7,084 & 1,429 & 12,334 \\
            \midrule
            \multirow{4}{*}{NELL-995}
            & v1 & 14 & 3,103 & 4,687 & 4,687 & 3,103 & 414 & 4,687 & 225 & 100 & 833 \\
            & v2 & 86 & 2,564 & 15,262 & 8,219 & 8,219 & 922 & 8,219 & 2,086 & 476 & 4,586 \\
            & v3 & 142 & 4,647 & 16,393 & 16,393 & 4,647 & 1,851 & 16,393 & 3,566 & 809 & 8,048 \\
            & v4 & 76 & 2,092 & 7,546 & 7,546 & 2,092 & 876 & 7,546 & 2,795 & 7,073 & 731 \\
            \bottomrule
        \end{tabular}
    \end{adjustbox}
    \caption{Dataset Statistics for inductive knowledge graph reasoning datasets. In each split, one needs to infer \#Query triplets based \#Fact triplets.}
    \label{tab:app.tab.2}
\end{table}

%% file: tables/vs_ultra.tex
\begin{table*}[htb]
    \centering
    \begin{adjustbox}{width=.8\textwidth}
    
    \begin{tabular}{lcccccccccccc}
    \toprule
    
    \multicolumn{1}{c}{\multirow{2}[4]{*}{\bf{Method}}} & \multicolumn{3}{c}{\bf{v1}} & \multicolumn{3}{c}{\bf{v2}} & \multicolumn{3}{c}{\bf{v3}} & \multicolumn{3}{c}{\bf{v4}}   \\
    
    \cmidrule(r){2-4} \cmidrule(r){5-7} \cmidrule(r){8-10} \cmidrule(r){11-13} 
    & \multicolumn{1}{c}{MRR} & \multicolumn{1}{c}{H@1} & \multicolumn{1}{c}{H@10} & \multicolumn{1}{c}{MRR} & \multicolumn{1}{c}{H@1} & \multicolumn{1}{c}{H@10} & \multicolumn{1}{c}{MRR} & \multicolumn{1}{c}{H@1} & \multicolumn{1}{c}{H@10} & \multicolumn{1}{c}{MRR} & \multicolumn{1}{c}{H@1} & \multicolumn{1}{c}{H@10}   \\

    \midrule

    \multicolumn{13}{c}{\bf{WN18RR}} \\
    \midrule

    \textsc{Ultra}{\small~\cite{abs-2310-04562}} & 0.685 & - & 79.3 & 0.679 & - & 77.9 & 0.411 & - & 54.6 & 0.614 & - & 72.0 \\
    \hdashline
    \rowcolor{Tan!20}
    \bf{\model} & \bf{0.752} & \bf{71.5} & \bf{81.9} & \bf{0.709} & \bf{65.6} & \bf{81.7} & \bf{0.467} & \bf{40.6} & \bf{57.1} & \bf{0.646} & \bf{60.9} & \bf{72.7} \\

    \midrule
    \multicolumn{13}{c}{\bf{NELL-995}} \\
    \midrule
    
    \textsc{Ultra}{\small~\cite{abs-2310-04562}} & 0.757 & - & 87.8 & \bf{0.575} & - & \bf{76.1} & \bf{0.563} & - & \bf{75.5} & \bf{0.469} & - & \bf{73.3} \\
    \hdashline
    \rowcolor{Tan!20}
    \bf{\model} & \bf{0.827} & \bf{77.0} & \bf{93.0} & 0.465 & 35.7 & 65.7 & 0.478 & 37.8 & 65.7 & 0.378 & 26.7 & 59.8 \\
    
    \bottomrule
    \end{tabular}
    \end{adjustbox}
    
  \caption{Inductive knowledge graph reasoning performance comparison between \model and \textsc{Ultra}. The best results are \textbf{boldfaced}. Our proposed model, \model is marked by {\colorbox{Tan!20}{\textbf{tan}}}.}
  \label{tab:app.tab.3}
\end{table*}

%% file: tables/hom_link_prediction.tex
\begin{table*}[htb]
    \centering
    \begin{adjustbox}{max width=.7\textwidth}
    \begin{tabular}{lcccccc}
    \toprule
    
    \multicolumn{1}{c}{\multirow{2}[4]{*}{\bf{Method}}} & \multicolumn{2}{c}{\bf{Cora}} & \multicolumn{2}{c}{\bf{Citeseer}} & \multicolumn{2}{c}{\bf{PubMed}}   \\
    
    \cmidrule(r){2-3} \cmidrule(r){4-5} \cmidrule(r){6-7}  
    & \multicolumn{1}{c}{AUROC} & \multicolumn{1}{c}{AP} & \multicolumn{1}{c}{AUROC} & \multicolumn{1}{c}{AP} & \multicolumn{1}{c}{AUROC} & \multicolumn{1}{c}{AP}   \\

    \midrule

    Katz Index~\cite{katz1953new} & 0.834 & 0.889 & 0.768 & 0.810 & 0.757 & 0.856 \\
    Personalized PageRank~\cite{page1998pagerank} & 0.845 & 0.899 & 0.762 & 0.814 & 0.763 & 0.860 \\
    SimRank~\cite{JehW02} & 0.838 & 0.888 & 0.755 & 0.805 & 0.743 & 0.829 \\
    DeepWalk~\cite{PerozziAS14} & 0.831 & 0.850 & 0.805 & 0.836 & 0.844 & 0.841 \\
    LINE~\cite{TangQWZYM15} & 0.844 & 0.879 & 0.838 & 0.868 & 0.891 & 0.914 \\
    VGAE~\cite{KipfW16a} & 0.914 & 0.926 & 0.908 & 0.920 & 0.944 & 0.947 \\
    S-VGAE~\cite{DavidsonFCKT18} & 0.941 & 0.941 & \bf{0.947} & \bf{0.952} & 0.960 & 0.960 \\
    SEAL~\cite{ZhangC18} & 0.933 & 0.942 & 0.905 & 0.924 & 0.978 & 0.979 \\
    TLC-GNN~\cite{YanMGT021} & 0.934 & 0.931 & 0.909 & 0.916 & 0.970 & 0.968 \\
    NBFNet~\cite{ZhuZXT21} & 0.956 & 0.962 & 0.923 & 0.936 & 0.983 & 0.982 \\
    
    \hdashline
    \rowcolor{Tan!20}
    \bf{\model} & \bf{0.961} & \bf{0.965} & 0.941 & 0.950 & \bf{0.987} & \bf{0.988} \\
    
    \bottomrule
    \end{tabular}
    \end{adjustbox}
    
  \caption{Results of homogeneous graph link prediction task. The best results are \textbf{boldfaced}. Our proposed model, \model is marked by {\colorbox{Tan!20}{\textbf{tan}}}. Most baseline results are taken from original papers.}
  \label{tab:app.tab.4}
\end{table*}

%% file: tables/impact_of_layers.tex
\begin{table*}[htb]
    \centering
    \begin{adjustbox}{width=.6\textwidth}
    
    \begin{tabular}{cccccccccc}
    \toprule
    
    \multicolumn{1}{c}{\multirow{2}[4]{*}{\bf{\#Value Layers}}} & \multicolumn{3}{c}{\bf{\#Attention Layers=1}} & \multicolumn{3}{c}{\bf{\#Attention Layers=2}} & \multicolumn{3}{c}{\bf{\#Attention Layers=3}} \\
    
    \cmidrule(r){2-4} \cmidrule(r){5-7} \cmidrule(r){8-10} 
    & \multicolumn{1}{c}{MRR} & \multicolumn{1}{c}{H@1} & \multicolumn{1}{c}{H@10} & \multicolumn{1}{c}{MRR} & \multicolumn{1}{c}{H@1} & \multicolumn{1}{c}{H@10} & \multicolumn{1}{c}{MRR} & \multicolumn{1}{c}{H@1} & \multicolumn{1}{c}{H@10} \\

    \midrule

    \multicolumn{10}{c}{\bf{\#Query Layers}=1} \\
    \midrule
    1 & \cellcolor{YellowGreen!10}{0.374} & \cellcolor{Cyan!10}{36.4} & \cellcolor{Goldenrod!10}{39.4} & \cellcolor{YellowGreen!20}{0.442} & \cellcolor{Cyan!20}{41.3} & \cellcolor{Goldenrod!20}{49.5} & \cellcolor{YellowGreen!80}{0.543} & \cellcolor{Cyan!80}{50.1} & \cellcolor{Goldenrod!80}{62.5}  \\
    2 & \cellcolor{YellowGreen!20}{0.425} & \cellcolor{Cyan!20}{40.5} & \cellcolor{Goldenrod!20}{46.1} & \cellcolor{YellowGreen!60}{0.554} & \cellcolor{Cyan!40}{50.7} & \cellcolor{Goldenrod!40}{65.0} & \cellcolor{YellowGreen!80}{0.572} & \cellcolor{Cyan!80}{52.0} & \cellcolor{Goldenrod!80}{67.6}  \\
    3 & \cellcolor{YellowGreen!40}{0.538} & \cellcolor{Cyan!40}{49.7} & \cellcolor{Goldenrod!40}{62.0} & \cellcolor{YellowGreen!80}{0.573} & \cellcolor{Cyan!80}{52.0} & \cellcolor{Goldenrod!80}{67.0} & \cellcolor{YellowGreen!80}{0.578} & \cellcolor{Cyan!80}{52.4} & \cellcolor{Goldenrod!80}{67.8} \\
    
    \midrule

    \multicolumn{10}{c}{\bf{\#Query Layers}=2} \\
    \midrule
    1 & \cellcolor{YellowGreen!10}{0.376} & \cellcolor{Cyan!10}{36.4} & \cellcolor{Goldenrod!10}{39.6} & \cellcolor{YellowGreen!40}{0.540} & \cellcolor{Cyan!40}{50.0} & \cellcolor{Goldenrod!40}{61.9} & \cellcolor{YellowGreen!80}{0.572} & \cellcolor{Cyan!80}{52.1} & \cellcolor{Goldenrod!80}{66.8}  \\
    2 & \cellcolor{YellowGreen!20}{0.425} & \cellcolor{Cyan!20}{40.5} & \cellcolor{Goldenrod!20}{46.0} & \cellcolor{YellowGreen!60}{0.551} & \cellcolor{Cyan!40}{50.7} & \cellcolor{Goldenrod!40}{62.1} & \cellcolor{YellowGreen!80}{0.576} & \cellcolor{Cyan!80}{52.5} & \cellcolor{Goldenrod!80}{67.5}  \\
    3 & \cellcolor{YellowGreen!40}{0.539} & \cellcolor{Cyan!40}{49.8} & \cellcolor{Goldenrod!40}{61.9} & \cellcolor{YellowGreen!80}{0.572} & \cellcolor{Cyan!80}{51.5} & \cellcolor{Goldenrod!80}{67.2} & \cellcolor{YellowGreen!80}{0.578} & \cellcolor{Cyan!80}{52.2} & \cellcolor{Goldenrod!80}{68.1} \\
    
    \midrule

    \multicolumn{10}{c}{\bf{\#Query Layers}=3} \\
    \midrule
    1 & \cellcolor{YellowGreen!10}{0.379} & \cellcolor{Cyan!10}{37.0} & \cellcolor{Goldenrod!10}{39.7} & \cellcolor{YellowGreen!60}{0.554} & \cellcolor{Cyan!40}{50.7} & \cellcolor{Goldenrod!60}{64.6} & \cellcolor{YellowGreen!80}{0.568} & \cellcolor{Cyan!80}{51.7} & \cellcolor{Goldenrod!80}{66.8} \\
    2 & \cellcolor{YellowGreen!20}{0.428} & \cellcolor{Cyan!20}{40.9} & \cellcolor{Goldenrod!20}{46.2} & \cellcolor{YellowGreen!80}{0.569} & \cellcolor{Cyan!80}{51.8} & \cellcolor{Goldenrod!80}{67.1} & \cellcolor{YellowGreen!80}{0.575} & \cellcolor{Cyan!80}{52.0} & \cellcolor{Goldenrod!80}{67.9} \\
    3 & \cellcolor{YellowGreen!40}{0.537} & \cellcolor{Cyan!40}{49.6} & \cellcolor{Goldenrod!40}{61.9} & \cellcolor{YellowGreen!80}{0.572} & \cellcolor{Cyan!80}{52.3} & \cellcolor{Goldenrod!80}{67.0} & \cellcolor{YellowGreen!80}{0.577} & \cellcolor{Cyan!80}{51.9} & \cellcolor{Goldenrod!80}{67.3} \\

    \bottomrule
    \end{tabular}
    \end{adjustbox}
    
  \caption{Results of different numbers of layers on the WN18RR dataset, with each metric represented by a different color. The darkness or lightness of the color corresponds to the performance of the metric.}
  \label{tab:app.tab.5}
  \vspace{-0.5cm}
\end{table*}

%% file: main.bbl
\begin{thebibliography}{99}
\providecommand{\natexlab}[1]{#1}
\providecommand{\url}[1]{\texttt{#1}}
\expandafter\ifx\csname urlstyle\endcsname\relax
  \providecommand{\doi}[1]{doi: #1}\else
  \providecommand{\doi}{doi: \begingroup \urlstyle{rm}\Url}\fi

\bibitem[Abujabal et~al.(2018)Abujabal, Roy, Yahya, and Weikum]{AbujabalRYW18}
Abujabal, A., Roy, R.~S., Yahya, M., and Weikum, G.
\newblock Never-ending learning for open-domain question answering over knowledge bases.
\newblock In \emph{{WWW}}, pp.\  1053--1062. {ACM}, 2018.

\bibitem[Alon \& Yahav(2021)Alon and Yahav]{0002Y21}
Alon, U. and Yahav, E.
\newblock On the bottleneck of graph neural networks and its practical implications.
\newblock In \emph{{ICLR}}. OpenReview.net, 2021.

\bibitem[Ba et~al.(2016)Ba, Kiros, and Hinton]{BaKH16}
Ba, L.~J., Kiros, J.~R., and Hinton, G.~E.
\newblock Layer normalization.
\newblock \emph{CoRR}, abs/1607.06450, 2016.

\bibitem[Barcel{\'{o}} et~al.(2022)Barcel{\'{o}}, Galkin, Morris, and Orth]{Barcelo00O22}
Barcel{\'{o}}, P., Galkin, M., Morris, C., and Orth, M. A.~R.
\newblock Weisfeiler and leman go relational.
\newblock In \emph{LoG}, volume 198 of \emph{Proceedings of Machine Learning Research}, pp.\ ~46. {PMLR}, 2022.

\bibitem[Bordes et~al.(2013)Bordes, Usunier, Garc{\'{\i}}a{-}Dur{\'{a}}n, Weston, and Yakhnenko]{BordesUGWY13}
Bordes, A., Usunier, N., Garc{\'{\i}}a{-}Dur{\'{a}}n, A., Weston, J., and Yakhnenko, O.
\newblock Translating embeddings for modeling multi-relational data.
\newblock In \emph{{NeurIPS}}, pp.\  2787--2795, 2013.

\bibitem[Cao et~al.(2019)Cao, Wang, He, Hu, and Chua]{0003W0HC19}
Cao, Y., Wang, X., He, X., Hu, Z., and Chua, T.
\newblock Unifying knowledge graph learning and recommendation: Towards a better understanding of user preferences.
\newblock In \emph{{WWW}}, pp.\  151--161. {ACM}, 2019.

\bibitem[Chamberlain et~al.(2023)Chamberlain, Shirobokov, Rossi, Frasca, Markovich, Hammerla, Bronstein, and Hansmire]{ChamberlainSRFM23}
Chamberlain, B.~P., Shirobokov, S., Rossi, E., Frasca, F., Markovich, T., Hammerla, N.~Y., Bronstein, M.~M., and Hansmire, M.
\newblock Graph neural networks for link prediction with subgraph sketching.
\newblock In \emph{{ICLR}}. OpenReview.net, 2023.

\bibitem[Chen et~al.(2022)Chen, O'Bray, and Borgwardt]{ChenOB22}
Chen, D., O'Bray, L., and Borgwardt, K.~M.
\newblock Structure-aware transformer for graph representation learning.
\newblock In \emph{{ICML}}, volume 162 of \emph{Proceedings of Machine Learning Research}, pp.\  3469--3489. {PMLR}, 2022.

\bibitem[Chen et~al.(2021)Chen, Liu, Gao, Jiao, Zhang, and Ji]{ChenLG0ZJ21}
Chen, S., Liu, X., Gao, J., Jiao, J., Zhang, R., and Ji, Y.
\newblock Hitter: Hierarchical transformers for knowledge graph embeddings.
\newblock In \emph{{EMNLP}}, pp.\  10395--10407. Association for Computational Linguistics, 2021.

\bibitem[Cheng et~al.(2020)Cheng, Yang, Wang, Zhang, and Zhang]{ChengYWZ020}
Cheng, D., Yang, F., Wang, X., Zhang, Y., and Zhang, L.
\newblock Knowledge graph-based event embedding framework for financial quantitative investments.
\newblock In \emph{{SIGIR}}, pp.\  2221--2230. {ACM}, 2020.

\bibitem[Cui \& Chen(2022)Cui and Chen]{CuiC22}
Cui, W. and Chen, X.
\newblock Instance-based learning for knowledge base completion.
\newblock In \emph{NeurIPS}, pp.\  30744--30755, 2022.

\bibitem[Dass et~al.(2023)Dass, Wu, Shi, Li, Ye, Wang, and Lin]{DassWSLYWL23}
Dass, J., Wu, S., Shi, H., Li, C., Ye, Z., Wang, Z., and Lin, Y.
\newblock Vitality: Unifying low-rank and sparse approximation for vision transformer acceleration with a linear taylor attention.
\newblock In \emph{{HPCA}}, pp.\  415--428. {IEEE}, 2023.

\bibitem[Davidson et~al.(2018)Davidson, Falorsi, Cao, Kipf, and Tomczak]{DavidsonFCKT18}
Davidson, T.~R., Falorsi, L., Cao, N.~D., Kipf, T., and Tomczak, J.~M.
\newblock Hyperspherical variational auto-encoders.
\newblock In \emph{{UAI}}, pp.\  856--865. {AUAI} Press, 2018.

\bibitem[Dettmers et~al.(2018)Dettmers, Minervini, Stenetorp, and Riedel]{DettmersMS018}
Dettmers, T., Minervini, P., Stenetorp, P., and Riedel, S.
\newblock Convolutional 2d knowledge graph embeddings.
\newblock In \emph{{AAAI}}, pp.\  1811--1818. {AAAI} Press, 2018.

\bibitem[Devlin et~al.(2019)Devlin, Chang, Lee, and Toutanova]{DevlinCLT19}
Devlin, J., Chang, M., Lee, K., and Toutanova, K.
\newblock {BERT:} pre-training of deep bidirectional transformers for language understanding.
\newblock In \emph{{NAACL-HLT}}, pp.\  4171--4186. Association for Computational Linguistics, 2019.

\bibitem[Dosovitskiy et~al.(2021)Dosovitskiy, Beyer, Kolesnikov, Weissenborn, Zhai, Unterthiner, Dehghani, Minderer, Heigold, Gelly, Uszkoreit, and Houlsby]{DosovitskiyB0WZ21}
Dosovitskiy, A., Beyer, L., Kolesnikov, A., Weissenborn, D., Zhai, X., Unterthiner, T., Dehghani, M., Minderer, M., Heigold, G., Gelly, S., Uszkoreit, J., and Houlsby, N.
\newblock An image is worth 16x16 words: Transformers for image recognition at scale.
\newblock In \emph{{ICLR}}. OpenReview.net, 2021.

\bibitem[Dwivedi \& Bresson(2020)Dwivedi and Bresson]{abs-2012-09699}
Dwivedi, V.~P. and Bresson, X.
\newblock A generalization of transformer networks to graphs.
\newblock \emph{CoRR}, abs/2012.09699, 2020.

\bibitem[Franceschi et~al.(2019)Franceschi, Niepert, Pontil, and He]{FranceschiNPH19}
Franceschi, L., Niepert, M., Pontil, M., and He, X.
\newblock Learning discrete structures for graph neural networks.
\newblock In \emph{{ICML}}, volume~97 of \emph{Proceedings of Machine Learning Research}, pp.\  1972--1982. {PMLR}, 2019.

\bibitem[Galkin et~al.(2024)Galkin, Yuan, Mostafa, Tang, and Zhu]{abs-2310-04562}
Galkin, M., Yuan, X., Mostafa, H., Tang, J., and Zhu, Z.
\newblock Towards foundation models for knowledge graph reasoning.
\newblock In \emph{{ICLR}}. OpenReview.net, 2024.

\bibitem[Gilmer et~al.(2017)Gilmer, Schoenholz, Riley, Vinyals, and Dahl]{GilmerSRVD17}
Gilmer, J., Schoenholz, S.~S., Riley, P.~F., Vinyals, O., and Dahl, G.~E.
\newblock Neural message passing for quantum chemistry.
\newblock In \emph{{ICML}}, volume~70 of \emph{Proceedings of Machine Learning Research}, pp.\  1263--1272. {PMLR}, 2017.

\bibitem[Glorot et~al.(2011)Glorot, Bordes, and Bengio]{GlorotBB11}
Glorot, X., Bordes, A., and Bengio, Y.
\newblock Deep sparse rectifier neural networks.
\newblock In \emph{{AISTATS}}, volume~15 of \emph{{JMLR} Proceedings}, pp.\  315--323. JMLR.org, 2011.

\bibitem[Hamilton et~al.(2017)Hamilton, Ying, and Leskovec]{HamiltonYL17}
Hamilton, W.~L., Ying, Z., and Leskovec, J.
\newblock Inductive representation learning on large graphs.
\newblock In \emph{{NeurIPS}}, pp.\  1024--1034, 2017.

\bibitem[Hogan et~al.(2022)Hogan, Blomqvist, Cochez, d'Amato, de~Melo, Gutierrez, Kirrane, Gayo, Navigli, Neumaier, Ngomo, Polleres, Rashid, Rula, Schmelzeisen, Sequeda, Staab, and Zimmermann]{HoganBCdMGKGNNN21}
Hogan, A., Blomqvist, E., Cochez, M., d'Amato, C., de~Melo, G., Gutierrez, C., Kirrane, S., Gayo, J. E.~L., Navigli, R., Neumaier, S., Ngomo, A.~N., Polleres, A., Rashid, S.~M., Rula, A., Schmelzeisen, L., Sequeda, J.~F., Staab, S., and Zimmermann, A.
\newblock Knowledge graphs.
\newblock \emph{{ACM} Comput. Surv.}, 54\penalty0 (4):\penalty0 71:1--71:37, 2022.

\bibitem[Hornik(1991)]{Hornik91}
Hornik, K.
\newblock Approximation capabilities of multilayer feedforward networks.
\newblock \emph{Neural Networks}, 4\penalty0 (2):\penalty0 251--257, 1991.

\bibitem[Hornik et~al.(1989)Hornik, Stinchcombe, and White]{HornikSW89}
Hornik, K., Stinchcombe, M.~B., and White, H.
\newblock Multilayer feedforward networks are universal approximators.
\newblock \emph{Neural Networks}, 2\penalty0 (5):\penalty0 359--366, 1989.

\bibitem[Huang et~al.(2019)Huang, Zhang, Li, and Li]{HuangZLL19}
Huang, X., Zhang, J., Li, D., and Li, P.
\newblock Knowledge graph embedding based question answering.
\newblock In \emph{{WSDM}}, pp.\  105--113. {ACM}, 2019.

\bibitem[Huang et~al.(2023)Huang, Orth, Ceylan, and Barcelo]{huang2023theory}
Huang, X., Orth, M.~R., Ceylan, I.~I., and Barcelo, P.
\newblock A theory of link prediction via relational weisfeiler-leman on knowledge graphs.
\newblock In \emph{NeurIPS}, 2023.

\bibitem[Jeh \& Widom(2002)Jeh and Widom]{JehW02}
Jeh, G. and Widom, J.
\newblock Simrank: a measure of structural-context similarity.
\newblock In \emph{{KDD}}, pp.\  538--543. {ACM}, 2002.

\bibitem[Ji et~al.(2022)Ji, Pan, Cambria, Marttinen, and Yu]{JiPCMY22}
Ji, S., Pan, S., Cambria, E., Marttinen, P., and Yu, P.~S.
\newblock A survey on knowledge graphs: Representation, acquisition, and applications.
\newblock \emph{{IEEE} Trans. Neural Networks Learn. Syst.}, 33\penalty0 (2):\penalty0 494--514, 2022.

\bibitem[Katz(1953)]{katz1953new}
Katz, L.
\newblock A new status index derived from sociometric analysis.
\newblock \emph{Psychometrika}, 18\penalty0 (1):\penalty0 39--43, 1953.

\bibitem[Kim et~al.(2022)Kim, Nguyen, Min, Cho, Lee, Lee, and Hong]{KimNMCLLH22}
Kim, J., Nguyen, D., Min, S., Cho, S., Lee, M., Lee, H., and Hong, S.
\newblock Pure transformers are powerful graph learners.
\newblock In \emph{NeurIPS}, 2022.

\bibitem[Kingma \& Ba(2015)Kingma and Ba]{KingmaB14}
Kingma, D.~P. and Ba, J.
\newblock Adam: {A} method for stochastic optimization.
\newblock In \emph{{ICLR} (Poster)}, 2015.

\bibitem[Kipf \& Welling(2016)Kipf and Welling]{KipfW16a}
Kipf, T.~N. and Welling, M.
\newblock Variational graph auto-encoders.
\newblock \emph{CoRR}, abs/1611.07308, 2016.

\bibitem[Kok \& Domingos(2007)Kok and Domingos]{KokD07}
Kok, S. and Domingos, P.~M.
\newblock Statistical predicate invention.
\newblock In \emph{{ICML}}, volume 227 of \emph{{ACM} International Conference Proceeding Series}, pp.\  433--440. {ACM}, 2007.

\bibitem[Koren et~al.(2009)Koren, Bell, and Volinsky]{KorenBV09}
Koren, Y., Bell, R.~M., and Volinsky, C.
\newblock Matrix factorization techniques for recommender systems.
\newblock \emph{Computer}, 42\penalty0 (8):\penalty0 30--37, 2009.

\bibitem[Lacroix et~al.(2018)Lacroix, Usunier, and Obozinski]{LacroixUO18}
Lacroix, T., Usunier, N., and Obozinski, G.
\newblock Canonical tensor decomposition for knowledge base completion.
\newblock In \emph{{ICML}}, volume~80 of \emph{Proceedings of Machine Learning Research}, pp.\  2869--2878. {PMLR}, 2018.

\bibitem[Lee et~al.(2023)Lee, Chung, and Whang]{LeeCW23}
Lee, J., Chung, C., and Whang, J.~J.
\newblock Ingram: Inductive knowledge graph embedding via relation graphs.
\newblock In \emph{{ICML}}, volume 202 of \emph{Proceedings of Machine Learning Research}, pp.\  18796--18809. {PMLR}, 2023.

\bibitem[Li et~al.(2023)Li, Wang, and Mao]{LiWM23}
Li, J., Wang, Q., and Mao, Z.
\newblock Inductive relation prediction from relational paths and context with hierarchical transformers.
\newblock In \emph{{ICASSP}}, pp.\  1--5. {IEEE}, 2023.

\bibitem[Li et~al.(2020)Li, Su, Duan, and Zheng]{abs-2007-14902}
Li, R., Su, J., Duan, C., and Zheng, S.
\newblock Linear attention mechanism: An efficient attention for semantic segmentation.
\newblock \emph{CoRR}, abs/2007.14902, 2020.

\bibitem[Li et~al.(2022)Li, Zhao, Li, He, Wang, Liu, Sun, Wang, Deng, Shen, Xie, and Zhang]{Li0LH0LSWDSXZ22}
Li, R., Zhao, J., Li, C., He, D., Wang, Y., Liu, Y., Sun, H., Wang, S., Deng, W., Shen, Y., Xie, X., and Zhang, Q.
\newblock House: Knowledge graph embedding with householder parameterization.
\newblock In \emph{{ICML}}, volume 162 of \emph{Proceedings of Machine Learning Research}, pp.\  13209--13224. {PMLR}, 2022.

\bibitem[Liben{-}Nowell \& Kleinberg(2007)Liben{-}Nowell and Kleinberg]{Liben-NowellK07}
Liben{-}Nowell, D. and Kleinberg, J.~M.
\newblock The link-prediction problem for social networks.
\newblock \emph{J. Assoc. Inf. Sci. Technol.}, 58\penalty0 (7):\penalty0 1019--1031, 2007.

\bibitem[Liu et~al.(2022)Liu, Huang, Chen, and Suykens]{LiuHCS22}
Liu, F., Huang, X., Chen, Y., and Suykens, J. A.~K.
\newblock Random features for kernel approximation: {A} survey on algorithms, theory, and beyond.
\newblock \emph{{IEEE} Trans. Pattern Anal. Mach. Intell.}, 44\penalty0 (10):\penalty0 7128--7148, 2022.

\bibitem[Lv et~al.(2022)Lv, Lin, Cao, Hou, Li, Liu, Li, and Zhou]{LvL00LLLZ22}
Lv, X., Lin, Y., Cao, Y., Hou, L., Li, J., Liu, Z., Li, P., and Zhou, J.
\newblock Do pre-trained models benefit knowledge graph completion? {A} reliable evaluation and a reasonable approach.
\newblock In \emph{{ACL} (Findings)}, pp.\  3570--3581. Association for Computational Linguistics, 2022.

\bibitem[Ma et~al.(2023)Ma, Lin, Lim, Romero{-}Soriano, Dokania, Coates, Torr, and Lim]{Ma0LRDCTL23}
Ma, L., Lin, C., Lim, D., Romero{-}Soriano, A., Dokania, P.~K., Coates, M., Torr, P. H.~S., and Lim, S.
\newblock Graph inductive biases in transformers without message passing.
\newblock In \emph{{ICML}}, volume 202 of \emph{Proceedings of Machine Learning Research}, pp.\  23321--23337. {PMLR}, 2023.

\bibitem[Mahdisoltani et~al.(2015)Mahdisoltani, Biega, and Suchanek]{MahdisoltaniBS15}
Mahdisoltani, F., Biega, J., and Suchanek, F.~M.
\newblock {YAGO3:} {A} knowledge base from multilingual wikipedias.
\newblock In \emph{{CIDR}}. www.cidrdb.org, 2015.

\bibitem[Mai et~al.(2021)Mai, Zheng, Yang, and Hu]{MaiZY021}
Mai, S., Zheng, S., Yang, Y., and Hu, H.
\newblock Communicative message passing for inductive relation reasoning.
\newblock In \emph{{AAAI}}, pp.\  4294--4302. {AAAI} Press, 2021.

\bibitem[Mikolov et~al.(2013)Mikolov, Sutskever, Chen, Corrado, and Dean]{MikolovSCCD13}
Mikolov, T., Sutskever, I., Chen, K., Corrado, G.~S., and Dean, J.
\newblock Distributed representations of words and phrases and their compositionality.
\newblock In \emph{{NeurIPS}}, pp.\  3111--3119, 2013.

\bibitem[Min et~al.(2022)Min, Chen, Bian, Xu, Zhao, Huang, Zhao, Huang, Ananiadou, and Rong]{abs-2202-08455}
Min, E., Chen, R., Bian, Y., Xu, T., Zhao, K., Huang, W., Zhao, P., Huang, J., Ananiadou, S., and Rong, Y.
\newblock Transformer for graphs: An overview from architecture perspective.
\newblock \emph{CoRR}, abs/2202.08455, 2022.

\bibitem[Neelakantan et~al.(2015)Neelakantan, Roth, and McCallum]{NeelakantanRM15}
Neelakantan, A., Roth, B., and McCallum, A.
\newblock Compositional vector space models for knowledge base completion.
\newblock In \emph{{ACL}}, pp.\  156--166. The Association for Computer Linguistics, 2015.

\bibitem[Nguyen et~al.(2018)Nguyen, Nguyen, Nguyen, and Phung]{NguyenNNP18}
Nguyen, D.~Q., Nguyen, T.~D., Nguyen, D.~Q., and Phung, D.~Q.
\newblock A novel embedding model for knowledge base completion based on convolutional neural network.
\newblock In \emph{{NAACL-HLT}}, pp.\  327--333. Association for Computational Linguistics, 2018.

\bibitem[Page(1998)]{page1998pagerank}
Page, L.
\newblock The pagerank citation ranking: Bringing order to the web. technical report.
\newblock \emph{Stanford Digital Library Technologies Project}, 1998.

\bibitem[Paszke et~al.(2019)Paszke, Gross, Massa, Lerer, Bradbury, Chanan, Killeen, Lin, Gimelshein, Antiga, Desmaison, K{\"{o}}pf, Yang, DeVito, Raison, Tejani, Chilamkurthy, Steiner, Fang, Bai, and Chintala]{PaszkeGMLBCKLGA19}
Paszke, A., Gross, S., Massa, F., Lerer, A., Bradbury, J., Chanan, G., Killeen, T., Lin, Z., Gimelshein, N., Antiga, L., Desmaison, A., K{\"{o}}pf, A., Yang, E.~Z., DeVito, Z., Raison, M., Tejani, A., Chilamkurthy, S., Steiner, B., Fang, L., Bai, J., and Chintala, S.
\newblock Pytorch: An imperative style, high-performance deep learning library.
\newblock In \emph{NeurIPS}, pp.\  8024--8035, 2019.

\bibitem[Perozzi et~al.(2014)Perozzi, Al{-}Rfou, and Skiena]{PerozziAS14}
Perozzi, B., Al{-}Rfou, R., and Skiena, S.
\newblock Deepwalk: online learning of social representations.
\newblock In \emph{{KDD}}, pp.\  701--710. {ACM}, 2014.

\bibitem[Qiu et~al.(2019)Qiu, Zhang, Feng, Liao, Jiang, Lyu, Liu, and Zhao]{QiuZFLJLLZ19}
Qiu, D., Zhang, Y., Feng, X., Liao, X., Jiang, W., Lyu, Y., Liu, K., and Zhao, J.
\newblock Machine reading comprehension using structural knowledge graph-aware network.
\newblock In \emph{{EMNLP/IJCNLP}}, pp.\  5895--5900. Association for Computational Linguistics, 2019.

\bibitem[Qu et~al.(2021)Qu, Chen, Xhonneux, Bengio, and Tang]{QuCXBT21}
Qu, M., Chen, J., Xhonneux, L. A.~C., Bengio, Y., and Tang, J.
\newblock Rnnlogic: Learning logic rules for reasoning on knowledge graphs.
\newblock In \emph{{ICLR}}. OpenReview.net, 2021.

\bibitem[Ramp{\'{a}}sek et~al.(2022)Ramp{\'{a}}sek, Galkin, Dwivedi, Luu, Wolf, and Beaini]{RampasekGDLWB22}
Ramp{\'{a}}sek, L., Galkin, M., Dwivedi, V.~P., Luu, A.~T., Wolf, G., and Beaini, D.
\newblock Recipe for a general, powerful, scalable graph transformer.
\newblock In \emph{NeurIPS}, 2022.

\bibitem[Rives et~al.(2021)Rives, Meier, Sercu, Goyal, Lin, Liu, Guo, Ott, Zitnick, Ma, and Fergus]{RivesMSGLLGOZMF21}
Rives, A., Meier, J., Sercu, T., Goyal, S., Lin, Z., Liu, J., Guo, D., Ott, M., Zitnick, C.~L., Ma, J., and Fergus, R.
\newblock Biological structure and function emerge from scaling unsupervised learning to 250 million protein sequences.
\newblock \emph{Proc. Natl. Acad. Sci. {USA}}, 118\penalty0 (15):\penalty0 e2016239118, 2021.

\bibitem[Sadeghian et~al.(2019)Sadeghian, Armandpour, Ding, and Wang]{SadeghianADW19}
Sadeghian, A., Armandpour, M., Ding, P., and Wang, D.~Z.
\newblock {DRUM:} end-to-end differentiable rule mining on knowledge graphs.
\newblock In \emph{NeurIPS}, pp.\  15321--15331, 2019.

\bibitem[Saxena et~al.(2022)Saxena, Kochsiek, and Gemulla]{SaxenaKG22}
Saxena, A., Kochsiek, A., and Gemulla, R.
\newblock Sequence-to-sequence knowledge graph completion and question answering.
\newblock In \emph{{ACL}}, pp.\  2814--2828. Association for Computational Linguistics, 2022.

\bibitem[Schlichtkrull et~al.(2018)Schlichtkrull, Kipf, Bloem, van~den Berg, Titov, and Welling]{SchlichtkrullKB18}
Schlichtkrull, M.~S., Kipf, T.~N., Bloem, P., van~den Berg, R., Titov, I., and Welling, M.
\newblock Modeling relational data with graph convolutional networks.
\newblock In \emph{{ESWC}}, volume 10843 of \emph{Lecture Notes in Computer Science}, pp.\  593--607. Springer, 2018.

\bibitem[Sen et~al.(2008)Sen, Namata, Bilgic, Getoor, Gallagher, and Eliassi{-}Rad]{SenNBGGE08}
Sen, P., Namata, G., Bilgic, M., Getoor, L., Gallagher, B., and Eliassi{-}Rad, T.
\newblock Collective classification in network data.
\newblock \emph{{AI} Mag.}, 29\penalty0 (3):\penalty0 93--106, 2008.

\bibitem[Sinha \& Duchi(2016)Sinha and Duchi]{SinhaD16}
Sinha, A. and Duchi, J.~C.
\newblock Learning kernels with random features.
\newblock In \emph{{NeurIPS}}, pp.\  1298--1306, 2016.

\bibitem[Stanfill \& Waltz(1986)Stanfill and Waltz]{StanfillW86}
Stanfill, C. and Waltz, D.~L.
\newblock Toward memory-based reasoning.
\newblock \emph{Commun. {ACM}}, 29\penalty0 (12):\penalty0 1213--1228, 1986.

\bibitem[Sun et~al.(2019)Sun, Deng, Nie, and Tang]{SunDNT19}
Sun, Z., Deng, Z., Nie, J., and Tang, J.
\newblock Rotate: Knowledge graph embedding by relational rotation in complex space.
\newblock In \emph{{ICLR} (Poster)}. OpenReview.net, 2019.

\bibitem[Tang et~al.(2015)Tang, Qu, Wang, Zhang, Yan, and Mei]{TangQWZYM15}
Tang, J., Qu, M., Wang, M., Zhang, M., Yan, J., and Mei, Q.
\newblock {LINE:} large-scale information network embedding.
\newblock In \emph{{WWW}}, pp.\  1067--1077. {ACM}, 2015.

\bibitem[Tang et~al.(2020)Tang, Huang, Wang, He, and Zhou]{TangHWHZ20}
Tang, Y., Huang, J., Wang, G., He, X., and Zhou, B.
\newblock Orthogonal relation transforms with graph context modeling for knowledge graph embedding.
\newblock In \emph{{ACL}}, pp.\  2713--2722. Association for Computational Linguistics, 2020.

\bibitem[Teru et~al.(2020)Teru, Denis, and Hamilton]{TeruDH20}
Teru, K.~K., Denis, E.~G., and Hamilton, W.~L.
\newblock Inductive relation prediction by subgraph reasoning.
\newblock In \emph{{ICML}}, volume 119 of \emph{Proceedings of Machine Learning Research}, pp.\  9448--9457. {PMLR}, 2020.

\bibitem[Toutanova \& Chen(2015)Toutanova and Chen]{ToutanovaC15}
Toutanova, K. and Chen, D.
\newblock Observed versus latent features for knowledge base and text inference.
\newblock In \emph{{CVSC}}, pp.\  57--66. Association for Computational Linguistics, 2015.

\bibitem[Trouillon et~al.(2017)Trouillon, Dance, Gaussier, Welbl, Riedel, and Bouchard]{TrouillonDGWRB17}
Trouillon, T., Dance, C.~R., Gaussier, {\'{E}}., Welbl, J., Riedel, S., and Bouchard, G.
\newblock Knowledge graph completion via complex tensor factorization.
\newblock \emph{J. Mach. Learn. Res.}, 18:\penalty0 130:1--130:38, 2017.

\bibitem[Tsai et~al.(2019)Tsai, Bai, Yamada, Morency, and Salakhutdinov]{TsaiBYMS19}
Tsai, Y.~H., Bai, S., Yamada, M., Morency, L., and Salakhutdinov, R.
\newblock Transformer dissection: An unified understanding for transformer's attention via the lens of kernel.
\newblock In \emph{{EMNLP/IJCNLP}}, pp.\  4343--4352. Association for Computational Linguistics, 2019.

\bibitem[Vashishth et~al.(2020)Vashishth, Sanyal, Nitin, and Talukdar]{VashishthSNT20}
Vashishth, S., Sanyal, S., Nitin, V., and Talukdar, P.~P.
\newblock Composition-based multi-relational graph convolutional networks.
\newblock In \emph{{ICLR}}. OpenReview.net, 2020.

\bibitem[Vaswani et~al.(2017)Vaswani, Shazeer, Parmar, Uszkoreit, Jones, Gomez, Kaiser, and Polosukhin]{VaswaniSPUJGKP17}
Vaswani, A., Shazeer, N., Parmar, N., Uszkoreit, J., Jones, L., Gomez, A.~N., Kaiser, L., and Polosukhin, I.
\newblock Attention is all you need.
\newblock In \emph{{NeurIPS}}, pp.\  5998--6008, 2017.

\bibitem[Wang et~al.(2022{\natexlab{a}})Wang, Zhou, Pan, Dong, Song, and Sha]{WangZPDSS22}
Wang, C., Zhou, X., Pan, S., Dong, L., Song, Z., and Sha, Y.
\newblock Exploring relational semantics for inductive knowledge graph completion.
\newblock In \emph{{AAAI}}, pp.\  4184--4192. {AAAI} Press, 2022{\natexlab{a}}.

\bibitem[Wang et~al.(2021)Wang, Ren, and Leskovec]{0004RL21}
Wang, H., Ren, H., and Leskovec, J.
\newblock Relational message passing for knowledge graph completion.
\newblock In \emph{{KDD}}, pp.\  1697--1707. {ACM}, 2021.

\bibitem[Wang et~al.(2022{\natexlab{b}})Wang, Zhao, Wei, and Liu]{0046ZWL22}
Wang, L., Zhao, W., Wei, Z., and Liu, J.
\newblock Simkgc: Simple contrastive knowledge graph completion with pre-trained language models.
\newblock In \emph{{ACL}}, pp.\  4281--4294. Association for Computational Linguistics, 2022{\natexlab{b}}.

\bibitem[Wang et~al.(2017)Wang, Mao, Wang, and Guo]{WangMWG17}
Wang, Q., Mao, Z., Wang, B., and Guo, L.
\newblock Knowledge graph embedding: {A} survey of approaches and applications.
\newblock \emph{{IEEE} Trans. Knowl. Data Eng.}, 29\penalty0 (12):\penalty0 2724--2743, 2017.

\bibitem[Wang et~al.(2023)Wang, Song, Wong, and See]{WangSWS23}
Wang, Z., Song, Y., Wong, G.~Y., and See, S.
\newblock Logical message passing networks with one-hop inference on atomic formulas.
\newblock In \emph{{ICLR}}. OpenReview.net, 2023.

\bibitem[Weisfeiler \& Leman(1968)Weisfeiler and Leman]{weisfeiler1968reduction}
Weisfeiler, B. and Leman, A.
\newblock The reduction of a graph to canonical form and the algebra which appears therein.
\newblock \emph{nti, Series}, 2\penalty0 (9):\penalty0 12--16, 1968.

\bibitem[Wu et~al.(2022)Wu, Zhao, Li, Wipf, and Yan]{WuZLWY22}
Wu, Q., Zhao, W., Li, Z., Wipf, D.~P., and Yan, J.
\newblock Nodeformer: {A} scalable graph structure learning transformer for node classification.
\newblock In \emph{NeurIPS}, 2022.

\bibitem[Wu et~al.(2023{\natexlab{a}})Wu, Yang, Zhao, He, Wipf, and Yan]{WuYZHWY23}
Wu, Q., Yang, C., Zhao, W., He, Y., Wipf, D., and Yan, J.
\newblock Difformer: Scalable (graph) transformers induced by energy constrained diffusion.
\newblock In \emph{{ICLR}}. OpenReview.net, 2023{\natexlab{a}}.

\bibitem[Wu et~al.(2023{\natexlab{b}})Wu, Zhao, Yang, Zhang, Nie, Jiang, Bian, and Yan]{wu2023simplifying}
Wu, Q., Zhao, W., Yang, C., Zhang, H., Nie, F., Jiang, H., Bian, Y., and Yan, J.
\newblock Simplifying and empowering transformers for large-graph representations.
\newblock In \emph{{NeurIPS}}, 2023{\natexlab{b}}.

\bibitem[Wu et~al.(2021)Wu, Jain, Wright, Mirhoseini, Gonzalez, and Stoica]{WuJWMGS21}
Wu, Z., Jain, P., Wright, M.~A., Mirhoseini, A., Gonzalez, J.~E., and Stoica, I.
\newblock Representing long-range context for graph neural networks with global attention.
\newblock In \emph{NeurIPS}, pp.\  13266--13279, 2021.

\bibitem[Xie et~al.(2022)Xie, Zhang, Li, Deng, Chen, Xiong, Chen, and Chen]{XieZLDCXCC22}
Xie, X., Zhang, N., Li, Z., Deng, S., Chen, H., Xiong, F., Chen, M., and Chen, H.
\newblock From discrimination to generation: Knowledge graph completion with generative transformer.
\newblock In \emph{{WWW} (Companion Volume)}, pp.\  162--165. {ACM}, 2022.

\bibitem[Xiong et~al.(2017)Xiong, Hoang, and Wang]{XiongHW17}
Xiong, W., Hoang, T., and Wang, W.~Y.
\newblock Deeppath: {A} reinforcement learning method for knowledge graph reasoning.
\newblock In \emph{{EMNLP}}, pp.\  564--573. Association for Computational Linguistics, 2017.

\bibitem[Xu et~al.(2019)Xu, Hu, Leskovec, and Jegelka]{XuHLJ19}
Xu, K., Hu, W., Leskovec, J., and Jegelka, S.
\newblock How powerful are graph neural networks?
\newblock In \emph{{ICLR}}. OpenReview.net, 2019.

\bibitem[Yan et~al.(2021)Yan, Ma, Gao, Tang, and Chen]{YanMGT021}
Yan, Z., Ma, T., Gao, L., Tang, Z., and Chen, C.
\newblock Link prediction with persistent homology: An interactive view.
\newblock In \emph{{ICML}}, volume 139 of \emph{Proceedings of Machine Learning Research}, pp.\  11659--11669. {PMLR}, 2021.

\bibitem[Yang et~al.(2015)Yang, Yih, He, Gao, and Deng]{YangYHGD14a}
Yang, B., Yih, W., He, X., Gao, J., and Deng, L.
\newblock Embedding entities and relations for learning and inference in knowledge bases.
\newblock In \emph{{ICLR} (Poster)}, 2015.

\bibitem[Yang et~al.(2017)Yang, Yang, and Cohen]{YangYC17}
Yang, F., Yang, Z., and Cohen, W.~W.
\newblock Differentiable learning of logical rules for knowledge base reasoning.
\newblock In \emph{{NeurIPS}}, pp.\  2319--2328, 2017.

\bibitem[Yao et~al.(2019)Yao, Mao, and Luo]{abs-1909-03193}
Yao, L., Mao, C., and Luo, Y.
\newblock {KG-BERT:} {BERT} for knowledge graph completion.
\newblock \emph{CoRR}, abs/1909.03193, 2019.

\bibitem[You et~al.(2021)You, Selman, Ying, and Leskovec]{YouGYL21}
You, J., Selman, J. M.~G., Ying, R., and Leskovec, J.
\newblock Identity-aware graph neural networks.
\newblock In \emph{{AAAI}}, pp.\  10737--10745. {AAAI} Press, 2021.

\bibitem[Zeng et~al.(2022)Zeng, Tu, Liu, Fu, and Su]{zeng2022toward}
Zeng, X., Tu, X., Liu, Y., Fu, X., and Su, Y.
\newblock Toward better drug discovery with knowledge graph.
\newblock \emph{Curr. Opin. Struct. Biol.}, 72:\penalty0 114--126, 2022.

\bibitem[Zhang \& Chen(2018)Zhang and Chen]{ZhangC18}
Zhang, M. and Chen, Y.
\newblock Link prediction based on graph neural networks.
\newblock In \emph{NeurIPS}, pp.\  5171--5181, 2018.

\bibitem[Zhang \& Chen(2020)Zhang and Chen]{ZhangC20}
Zhang, M. and Chen, Y.
\newblock Inductive matrix completion based on graph neural networks.
\newblock In \emph{{ICLR}}. OpenReview.net, 2020.

\bibitem[Zhang et~al.(2021)Zhang, Li, Xia, Wang, and Jin]{ZhangLXWJ21}
Zhang, M., Li, P., Xia, Y., Wang, K., and Jin, L.
\newblock Labeling trick: {A} theory of using graph neural networks for multi-node representation learning.
\newblock In \emph{NeurIPS}, pp.\  9061--9073, 2021.

\bibitem[Zhang \& Yao(2022)Zhang and Yao]{ZhangY22}
Zhang, Y. and Yao, Q.
\newblock Knowledge graph reasoning with relational digraph.
\newblock In \emph{{WWW}}, pp.\  912--924. {ACM}, 2022.

\bibitem[Zhang et~al.(2023)Zhang, Zhou, Yao, Chu, and Han]{ZhangZY0023}
Zhang, Y., Zhou, Z., Yao, Q., Chu, X., and Han, B.
\newblock Adaprop: Learning adaptive propagation for graph neural network based knowledge graph reasoning.
\newblock In \emph{{KDD}}, pp.\  3446--3457. {ACM}, 2023.

\bibitem[Zhang et~al.(2020)Zhang, Cai, Zhang, and Wang]{ZhangCZW20}
Zhang, Z., Cai, J., Zhang, Y., and Wang, J.
\newblock Learning hierarchy-aware knowledge graph embeddings for link prediction.
\newblock In \emph{{AAAI}}, pp.\  3065--3072. {AAAI} Press, 2020.

\bibitem[Zhu et~al.(2021)Zhu, Zhang, Xhonneux, and Tang]{ZhuZXT21}
Zhu, Z., Zhang, Z., Xhonneux, L. A.~C., and Tang, J.
\newblock Neural bellman-ford networks: {A} general graph neural network framework for link prediction.
\newblock In \emph{NeurIPS}, pp.\  29476--29490, 2021.

\bibitem[Zhu et~al.(2023)Zhu, Yuan, Galkin, Xhonneux, Zhang, Gazeau, and Tang]{zhu2023net}
Zhu, Z., Yuan, X., Galkin, M., Xhonneux, S., Zhang, M., Gazeau, M., and Tang, J.
\newblock A* net: A scalable path-based reasoning approach for knowledge graphs.
\newblock In \emph{{NeurIPS}}, 2023.

\end{thebibliography}
